\pdfoutput=1
\documentclass{article}

\usepackage[preprint]{neurips_2026}

\usepackage[utf8]{inputenc} %
\usepackage[T1]{fontenc}    %
\usepackage{hyperref}       %
\usepackage{url}            %
\usepackage{booktabs}       %
\usepackage{amsfonts}       %
\usepackage{nicefrac}       %
\usepackage{microtype}      %
\usepackage{xcolor}         %
\usepackage{wrapfig}
\usepackage{graphicx}
\usepackage{amsmath}
\usepackage{enumitem}
\usepackage{xspace}

\usepackage{multirow}
\usepackage{multicol}
\usepackage{tabularx}

\usepackage{amsthm}

\usepackage{booktabs}
\usepackage{multirow}
\usepackage[table]{xcolor}
\usepackage{wrapfig}

\usepackage{subcaption}
\usepackage[most]{tcolorbox}
\usepackage{amsmath}
\usepackage{amssymb}
\usepackage{amsthm}
\newtheorem{theorem}{Theorem}
\newtheorem{corollary}{Corollary}

\title{AgentVidBench: A Multi-Hop Video Question Answering Benchmark for Evaluating MLLM Agents}

\hypersetup{
    colorlinks = true,
    linkbordercolor = {white},
    linkcolor = blue!60!black,
    citecolor = blue!60!black,
    urlcolor = blue!60!black
}

\makeatletter
\def\blfootnote{\xdef\@thefnmark{}\@footnotetext}
\makeatother

\author{
  Seoyeon An\thanks{Equal contribution.} , Hyeonseo Jang\footnotemark[1] , Minsu Kim\footnotemark[1] , Chanho Lee, Younghan Park, Kangwook Lee
  \vspace{0.4em} \\
  KRAFTON
}

\begin{document}

\maketitle
\blfootnote{Correspondence to: Minsu Kim <minsu.kim@krafton.com>.}

\vspace{-0.55cm}
\begin{abstract}
Comprehensive video understanding is crucial for advancing artificial intelligence toward the intricate dynamics of the physical world. While recent advances in Multimodal Large Language Models (MLLMs) have demonstrated remarkable capabilities in video understanding, existing benchmarks remain confined to simple scene-level queries or global summaries that require only single-step inference. Real-world video understanding involves more challenging tasks that require multi-hop multimodal reasoning, and there is a critical absence of video benchmarks equipped to rigorously evaluate these agentic capabilities. To bridge this gap, we introduce \textbf{AgentVidBench}, a multi-hop video question answering benchmark focused on evaluating the spatial, temporal, and causal reasoning capabilities of MLLM agents. Beyond standard question-answer pairs, AgentVidBench provides step-by-step solution traces to support trajectory evaluation that assesses whether agents explicitly acquire the evidence needed to justify their answers. Experiments with 12 proprietary and open-source MLLMs show that single-turn performance remains limited on AgentVidBench, while integrating these models into state-of-the-art agentic workflows generally improves performance with respect to both accuracy and trajectory scores. We further present a simple yet effective agentic strategy that serves as a competitive baseline on AgentVidBench, establishing our benchmark as a holistic testbed for future research on agentic video understanding. Code and datasets are available at \url{https://github.com/krafton-ai/agentvidbench} and \url{https://huggingface.co/datasets/agentvidbench/agentvidbench}.

\end{abstract}

\vspace{-0.34cm}
\section{Introduction}

The rapid proliferation of video content across digital platforms has transformed it into the primary medium for capturing the intricate dynamics of the physical world\,\citep{DBLP:conf/cvpr/RenWG0KFJ25,DBLP:journals/corr/abs-2510-21447, DBLP:journals/corr/abs-2511-00062}. Unlike static images, video serves as a rich, high-dimensional repository of multimodal information, seamlessly integrating temporal sequences, spatial transformations, and synchronized audio-visual cues\,\citep{DBLP:conf/cvpr/CarreiraZ17, DBLP:conf/cvpr/GirdharELSAJM23}. This inherent complexity makes video understanding a pivotal frontier in artificial intelligence, as it requires models to go beyond simple pattern recognition and achieve a holistic comprehension of long-range dependencies and causal relationships\,\citep{DBLP:conf/cvpr/0004GGH18, DBLP:conf/cvpr/XiaoSYC21}. There has been a growing emphasis on the development of Multimodal Large Language Models (MLLMs), which harness the sophisticated reasoning capabilities of LLMs to decode and interpret diverse video signals.

Despite the remarkable potential of MLLMs, relying solely on their current architectural paradigms presents significant limitations in achieving human-level comprehension of complex video scenarios\,\citep{DBLP:conf/cvpr/0002WH00LWX0L0024, DBLP:conf/cvpr/FuDLLRZWZSZCLLZ25, DBLP:conf/nips/MangalamAM23}. 
The integration of agentic frameworks represents a prominent strategy to mitigate the inherent structural bottlenecks\,\citep{DBLP:conf/iclr/YaoZYDSN023, DBLP:journals/corr/abs-2303-11381, DBLP:conf/nips/0001ST00Z23}. Unlike traditional single-pass models, these agentic frameworks employ iterative reasoning, memory management, and external tool integration to interactively query and analyze video streams, thereby dynamically resolving spatio-temporal ambiguities\,\cite{DBLP:conf/eccv/WangZZY24, DBLP:conf/icml/YangCLW024}. 
However, there is a critical absence of complex video benchmarks explicitly designed to assess the true efficacy of multi-step agentic pipelines in solving compositional video understanding tasks.

\begin{figure}[t]
  \centering
  \includegraphics[width=0.95\textwidth, trim=0 0.2cm 0 0]{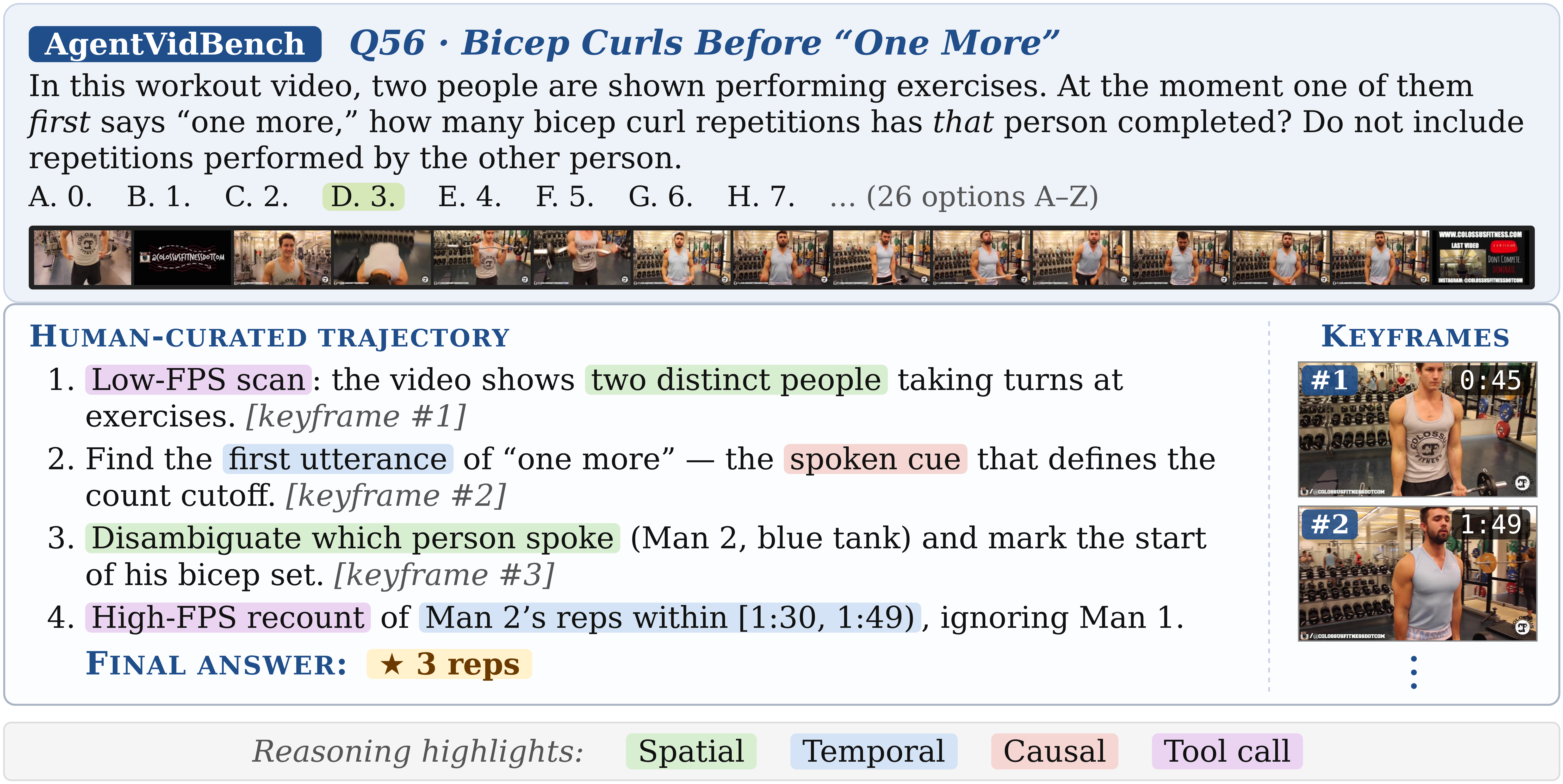}
  \caption{An illustrative example from AgentVidBench, requiring spatial, temporal, and causal reasoning. The example shows the question, answer choices (ground-truth answer highlighted), a full-video filmstrip overview, and a summarized human-curated trajectory.}
  \label{fig:illustrative-examples}
  \vspace{-0.6cm}
\end{figure}

To this end, we present \textbf{AgentVidBench} (Figure~\ref{fig:illustrative-examples}), a multi-hop video question answering benchmark specifically focused on evaluating the spatial, temporal, and causal video understanding capabilities of MLLM agents. We emphasize these three criteria as they reflect the necessary shift from passive, uniform video observation to active, interactive video exploration, akin to how humans naturally navigate and comprehend complex visual media. Specifically, solving complex video queries requires agents to explicitly zoom into specific regions for precise spatial reasoning, selectively focus on specific timeframes for temporal grounding, and actively traverse backward and forward across the timeline to infer causal links between disparate events. 

AgentVidBench comprises a meticulously curated set of 100 complex multiple-choice questions spanning 71 videos sourced from public web repositories. Specifically, the videos capture a broad spectrum of 12 fine-grained genres organized into 4 super-categories, with varying durations of up to 30 minutes. 
During the authoring process, we prohibited questions resolvable from a single frame or a localized timestamp lookup, ensuring each question is inherently compositional and demands a multi-step trajectory of active tool-use.
Every question is cross-validated by multiple human evaluators to confirm a single correct answer, and we systematically generate 25 challenging distractors to rigorously preclude models from relying on random chance. Unlike existing benchmarks that evaluate only the final output, our benchmark uniquely features detailed, step-by-step solution traces for each question to explicitly facilitate the assessment of multi-step agentic workflows. 

We evaluate various state-of-the-art MLLMs alongside their integration with agentic workflows using AgentVidBench. Specifically, our evaluation includes GPT-5\,\citep{DBLP:journals/corr/abs-2601-03267}, Claude Opus 4.7\,\citep{anthropic2026claudeopus47}, Gemini 2.5\,\citep{DBLP:journals/corr/abs-2507-06261}, Qwen3-VL\,\citep{DBLP:journals/corr/abs-2511-21631}, Qwen3.5\,\citep{qwen3.5}, and Gemma 4\,\citep{deepmind2026gemma4} for MLLMs, and STAR\,\citep{fan2025toolaugmented}, DVD\,\citep{DBLP:conf/nips/ZhangJGLLLL25}, and AVP\,\citep{wang2025activevideoperceptioniterative} for compatible agentic workflows. 
Experimental results on AgentVidBench demonstrate that conventional single-turn MLLMs struggle to solve multi-hop video QA tasks, whereas deploying them within state-of-the-art agentic workflows enhances their foundational performance, yielding improvements in both accuracy and trajectory scores. In addition, our ablation study reveals that adaptive video inspection across spatial, temporal, and causal reasoning capabilities is beneficial for complex video understanding. Based on these results, we suggest a simple yet effective agentic workflow that instantiates active evidence-seeking through controllable video inspection, serving as a strong baseline for future research. We believe that our findings on AgentVidBench not only highlight the strengths and vulnerabilities of existing models but also demonstrate how structured agentic workflows can significantly mitigate the inherent limitations of standard inference.

\textbf{Summary of Contributions:} (1) We introduce AgentVidBench, a multi-hop video question answering benchmark designed to evaluate the spatial, temporal, and causal reasoning capabilities of MLLM agents; (2) We extensively benchmark state-of-the-art MLLMs both as standalone models and within agentic workflows on AgentVidBench, demonstrating the necessity of multi-hop multimodal reasoning for video comprehension; (3) We propose a simple yet effective agentic strategy for fine-grained video navigation, and a novel evaluation pipeline to assess the quality of agent trajectories.

\section{Related Work}

\vspace{-0.2cm}
\paragraph{Video Question Answering.}
Video Question Answering (Video QA)\,\citep{DBLP:conf/emnlp/Zhong0X0DC22} aims to generate accurate natural language responses to queries based on dynamic video content, requiring a comprehensive understanding of spatial, temporal, and causal relationships\,\citep{DBLP:conf/mm/XuZX0Z0Z17, DBLP:conf/cvpr/JangSYKK17, DBLP:conf/mm/ShangLXJC21, DBLP:conf/sigir/YangFJWC21, DBLP:conf/cvpr/XiaoSYC21, DBLP:conf/aaai/XiaoYL0JC22}. 
Moving beyond traditional feature-extraction methods, the paradigm has shifted toward transformer-based architectures and Multimodal Large Language Models (MLLMs)\,\citep{DBLP:journals/monet/SunLLYWZ21, DBLP:journals/access/KhuranaD21, DBLP:journals/corr/abs-2404-16038, DBLP:journals/ijcv/XiaoHQLLZTYLCY25, DBLP:journals/tcsv/TangBXSLWZALZVHZLFZZLLX26}. 
Recently, the field has witnessed a surge in powerful foundation models as both proprietary (e.g., GPT-4o\,\citep{DBLP:journals/corr/abs-2410-21276}, Gemini 1.5\,\citep{DBLP:journals/corr/abs-2403-05530}) and open-source models (e.g., Qwen-3-VL\,\citep{DBLP:journals/corr/abs-2511-21631}, Gemma 4\,\citep{deepmind2026gemma4}) now exhibit unprecedented capabilities in video understanding.
However, they typically approach Video QA as a one-step generation process, fundamentally limiting their capacity to resolve complex queries that demand iterative, multi-step reasoning. To overcome this bottleneck, several studies focus on MLLM agents for video understanding by leveraging active temporal navigation (e.g., VideoAgent\,\citep{DBLP:conf/eccv/FanMWDLGL24}, Vgent\,\citep{DBLP:conf/nips/ShenZCE25}), hierarchical reasoning (e.g., VideoTree\,\citep{DBLP:conf/cvpr/WangYSYCBB25}, SiLVR\,\citep{DBLP:journals/tmlr/ZhangLWBB26}), and iterative visual discovery (e.g., VideoLucy\,\citep{DBLP:conf/nips/ZuoDKYJZSPLG25}, LVAgent\,\citep{DBLP:conf/iccv/ChenYCWLLW25}, DVD\,\citep{DBLP:conf/nips/ZhangJGLLLL25}).
Despite their immense promise, the evaluation of these agentic systems is largely confined to traditional benchmarks comprising simple video queries\,\citep{DBLP:conf/eccv/WangZZY24, DBLP:journals/corr/abs-2504-20091, DBLP:journals/corr/abs-2503-16032, DBLP:journals/corr/abs-2506-18071, DBLP:conf/visapp/MontesL26}, which fail to adequately measure their autonomous planning and iterative reasoning capabilities. 

\vspace{-0.2cm}
\paragraph{Evaluation Benchmarks for MLLMs.}
To evaluate the growing capabilities of MLLMs, a variety of Video QA benchmarks have been established. 
Traditional Video QA datasets primarily focus on basic action recognition and object tracking within short video snippets\,\citep{DBLP:conf/mm/XuZX0Z0Z17, DBLP:conf/aaai/YuXYYZZT19}.
Subsequent efforts expanded the scope into diverse domains and specialized reasoning tasks, introducing datasets like MovieQA\,\citep{DBLP:conf/cvpr/TapaswiZSTUF16} and TVQA\,\citep{DBLP:conf/emnlp/LeiYBB18} for narrative understanding, TGIF-QA\,\citep{DBLP:conf/cvpr/JangSYKK17} for spatio-temporal reasoning in short animations, NExT-QA\,\citep{DBLP:conf/cvpr/XiaoSYC21} for causal and temporal explanations, and MarioQA\,\citep{DBLP:conf/iccv/MunSJH17} for gameplay scenarios. 
Building on this foundation to address the necessity for even deeper cognitive evaluation, recent comprehensive benchmarks such as EgoSchema\,\citep{DBLP:conf/nips/MangalamAM23}, MVBench\,\citep{DBLP:conf/cvpr/0002WH00LWX0L0024}, Video-MME\,\citep{DBLP:conf/cvpr/FuDLLRZWZSZCLLZ25}, and Video-MME-v2\,\citep{DBLP:journals/corr/abs-2604-05015} have been introduced. 
Although these benchmarking works only provide the performance of MLLMs, we evaluate both MLLMs and advanced MLLM agents on our complex multi-hop video question answering dataset, thereby explicitly quantifying the performance gap between passive perception and autonomous, multi-step reasoning. In addition, the inclusion of detailed step-by-step solution traces allows our benchmark to evaluate not only the final answer accuracy but also the intrinsic quality and logical coherence of the MLLM agent trajectories. We provide more detailed comparisons of AgentVidBench with existing benchmarks in Appendix~\ref{sec:comparison_agentvidbench}.

\vspace{-0.2cm}
\paragraph{Evaluation Benchmarks for MLLM Agents.}
The evaluation of MLLMs as autonomous agents has recently gained substantial traction, leading to the development of agent-centric benchmarks. Existing frameworks, such as WebArena\,\citep{DBLP:conf/iclr/ZhouX0ZLSCOBF0N24}, VisualWebArena\,\citep{DBLP:conf/acl/KohLJDLHNZSF24}, and OSWorld\,\citep{DBLP:conf/nips/XieZCLZCHCSLLXZ24}, successfully evaluate multimodal agents situated in digital environments, focusing on tasks like web navigation and operating system control. In the visual and physical domains, benchmarks like EgoPlan\,\citep{DBLP:journals/ijcv/ChenGGDLWXSL26} emphasize embodied task planning from egocentric viewpoints. However, there remains a distinct scarcity of agent benchmarks specifically tailored for multi-hop Video QA. Our proposed AgentVidBench bridges this gap by formulating Video QA as an interactive environment, directly measuring MLLM agents' ability to perform autonomous, multi-step reasoning over dense temporal dynamics.

\vspace{-0.2cm}
\section{AgentVidBench}
We introduce AgentVidBench, a benchmark designed to evaluate MLLM agents through 100 complex multiple-choice questions over 71 videos.
These videos encompass a broad spectrum of 12 distinct genres, with durations ranging from a few seconds to several tens of minutes, thereby capturing a wide variety of dynamic visual contexts. Each question is paired with 26 answer options labeled A--Z, so that random guessing yields an expected accuracy of 3.85\%. In the following sections, we describe the theoretical motivation of AgentVidBench (Section~\ref{sec:benchmark-theory}) and how the dataset including videos, questions, answer choices, milestones, and metadata was collected (Section~\ref{sec:benchmark-construction}). We also provide detailed statistics of the dataset (Section~\ref{sec:statistics}), followed by our comprehensive evaluation pipeline for assessing both standard accuracy and agent trajectory scores (Section~\ref{sec:evaluation-pipeline}).

\vspace{-0.2cm}
\subsection{Theoretical Motivation}
\label{sec:benchmark-theory}
Standard single-pass MLLMs inherit the fixed computational depth of their underlying transformer architectures. Consequently, these models possess a bounded capacity for sequential computation within a single inference pass. 

\begin{corollary}[Informal; Fixed-Depth Limitation on Multi-Hop Reasoning]
\label{cor:informal_fixed_depth}
For any causal fixed-depth transformer with bounded precision, width, and number of attention heads that satisfies the query-independent attention assumption of Yao et al.~\citep{DBLP:conf/emnlp/YaoDZHK25}, there exists a sufficiently large $k$ and a worst-case $k$-hop reasoning task that the model fundamentally cannot solve in a single forward pass.
\vspace{-0.2cm}
\end{corollary}
This corollary directly follows from Theorem~5.1 of Yao et al.~\citep{DBLP:conf/emnlp/YaoDZHK25}, and the formal statement and proof are provided in Appendix~\ref{sec:theoretical_details}. This limitation aligns with a broader body of literature studying the limits of fixed-depth transformers on compositional and multi-hop reasoning~\citep{DBLP:conf/nips/DziriLSLJLWWB0H23, DBLP:conf/emnlp/YaoDZHK25, DBLP:conf/nips/YehudaiAB25}. 
A well-studied way to alleviate this limitation is to allocate additional sequential computation at inference time, which is one form of test-time scaling~\citep{DBLP:conf/nips/Wei0SBIXCLZ22, DBLP:journals/corr/abs-2412-16720, DBLP:journals/nature/GuoYZSWZXZMBZY025}. 
In the context of video understanding, we define a ``hop'' as a single dependent reasoning step or evidence-update operation, with $k$ denoting the total number of such sequential steps required to resolve a complex query. 
Conventional test-time scaling is limited for multi-hop video questions, as it relies on additional reasoning over a fixed visual context and lacks the capacity for adaptive observation to sequentially update evidence from the video.
Motivated by this bottleneck, we design AgentVidBench to evaluate models on questions requiring multiple steps of dependent reasoning and evidence-acquisition, thereby demonstrating the necessity of integrating agentic frameworks.

\vspace{-0.2cm}
\subsection{Benchmark Construction Pipeline}
\label{sec:benchmark-construction}

\begin{figure}[t]
  \centering
  \includegraphics[width=0.85\textwidth, trim=0 0.2cm 0 0]{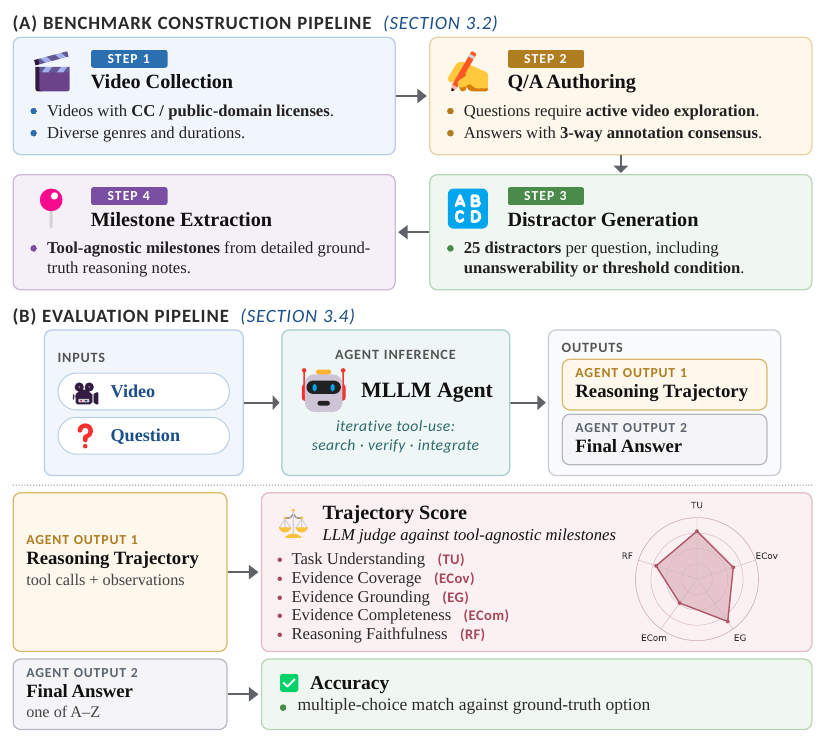}
  \vspace{-0.1cm}
  \caption{
  Overview of AgentVidBench.
  (a) Benchmark construction (Section~\ref{sec:benchmark-construction}): After video collection, questions are authored with consensus-based answers, distractors are generated to form multiple-choice options, and tool-agnostic milestones are extracted from reasoning notes.
  (b) Evaluation (Section~\ref{sec:evaluation-pipeline}): MLLM agent produces a reasoning trajectory and a final answer, which are assessed by a trajectory score and accuracy, respectively.
  }
  \label{fig:dataset-overview}
  \vspace{-0.4cm}
\end{figure}

\paragraph{Video Collection.}
To ensure strict compliance with copyright regulations and guarantee the public availability of our benchmark, we source videos from both publicly accessible online sources and our own recordings. Among the 71 videos in our benchmark, externally sourced videos are drawn from platforms such as YouTube and Wikimedia Commons, as well as public-domain and openly licensed media repositories, including archives maintained by government and public-sector organizations such as NASA.
The full list of video sources and licenses is available in Appendix~\ref{app:video-sources}. For every externally sourced video, we manually verified that its Creative Commons license or public-domain status allows us to include the video file in the benchmark package. 

\vspace{-0.1cm}

\paragraph{Question-Answer Annotation.}
To guarantee the quality of our benchmark, all questions and answers are manually designed, written, and validated through a structured authoring workflow with five human annotators. We organize the sequential authoring workflow by functional roles as follows:

\begin{itemize}[leftmargin=*]
\setlength\itemsep{0cm}
\vspace{-0.1cm}
\item \textbf{Question Annotation}:
Human annotators manually formulate questions for each video through comprehensive and iterative reviews of the video content. In the authoring process, annotators seek to cover a diverse set of skills spanning perception, counting, temporal and spatial reasoning, audio understanding, and multi-step agentic reasoning. Each question is written to necessitate active video exploration, such as revisiting specific moments or aggregating evidence across distant timestamps, while strictly maintaining a single verifiable correct answer.
\item \textbf{Answer Annotation}:
To ensure rigorous quality control, the answer to each question is determined by the absolute consensus of independent human annotators. Specifically, each question is initially answered by three annotators, and they compare their responses with each other. If their responses align perfectly, the consensus answer is finalized as the ground truth. However, if they identify potential issues such as answer discrepancies or ambiguous phrasing, the questions for these flagged instances are refined until agreement on a singular correct answer is achieved. As an additional validation step, we analyze preliminary model responses to uncover latent edge cases not anticipated during human review. We provide more details on the review process and inter-annotator agreement in Appendix~\ref{app:annotation-process}.
\item \textbf{Distractor Generation}:
Following the annotation of the correct answer, we generate 25 distractors for each question based on the answer type. For example, numeric type answers are paired with nearby integer values that span a plausible range, while non-numeric type answers are paired with similar words to the correct answer or relevant candidates visible in the video. To prevent models from exploiting structural shortcuts, the candidate set may also include assertions of unanswerability (e.g., ``the requested name is not mentioned'') or threshold conditions (e.g., ``26+''), strategically serving as either the actual ground truth or deceptive distractors. We note that the generated distractors avoid semantic ambiguity, partial validity, and trivial identifiability.
\end{itemize}

\vspace{-0.1cm}

\paragraph{Milestone Extraction.}
To address the limitations of standard answer-only evaluation, we introduce milestone extraction that facilitates the rigorous assessment of agent trajectories. We define a milestone as an essential evidence-acquisition or reasoning step required to justify the correct answer. Crucially, milestones are formulated to be tool-agnostic in that they specify \emph{what} evidence is needed rather than \emph{how} an agent should obtain it, ensuring a fair comparison across diverse agentic frameworks with varying toolsets or search paths. The milestone extraction process begins with annotators drafting a detailed ground-truth reasoning note that documents the human video investigation process to reach the answer. This descriptive note is subsequently abstracted into a structured set of evidence milestones. During this abstraction, each milestone is categorized by an evidence type, such as temporal localization, OCR/text reading, visual recognition, counting, exclusion checking, arithmetic, or verification. We provide more details on the milestone extraction process in Appendix~\ref{app:evidence-milestones}. Finally, these structured milestones serve as the reference evidence for the multi-step trajectory evaluation, which is detailed in Section~\ref{sec:evaluation-pipeline}.

\vspace{-0.1cm}

\paragraph{Taxonomy Annotation.}
Each video and question in our benchmark is annotated with a structured set of taxonomy metadata to support fine-grained analysis. Video-level taxonomy axes include \emph{genre}, \emph{duration}, \emph{production style}, and \emph{audio type}, while question-level taxonomy axes include \emph{required skills} and \emph{difficulty}. We provide more details regarding the taxonomy annotation process and the complete list of taxonomy labels in Appendix~\ref{app:taxonomy}.

\subsection{Benchmark Statistics}
\label{sec:statistics}

\begin{figure}[t]
  \centering

  \begin{subfigure}{0.48\textwidth}
    \centering
    \includegraphics[width=\linewidth]{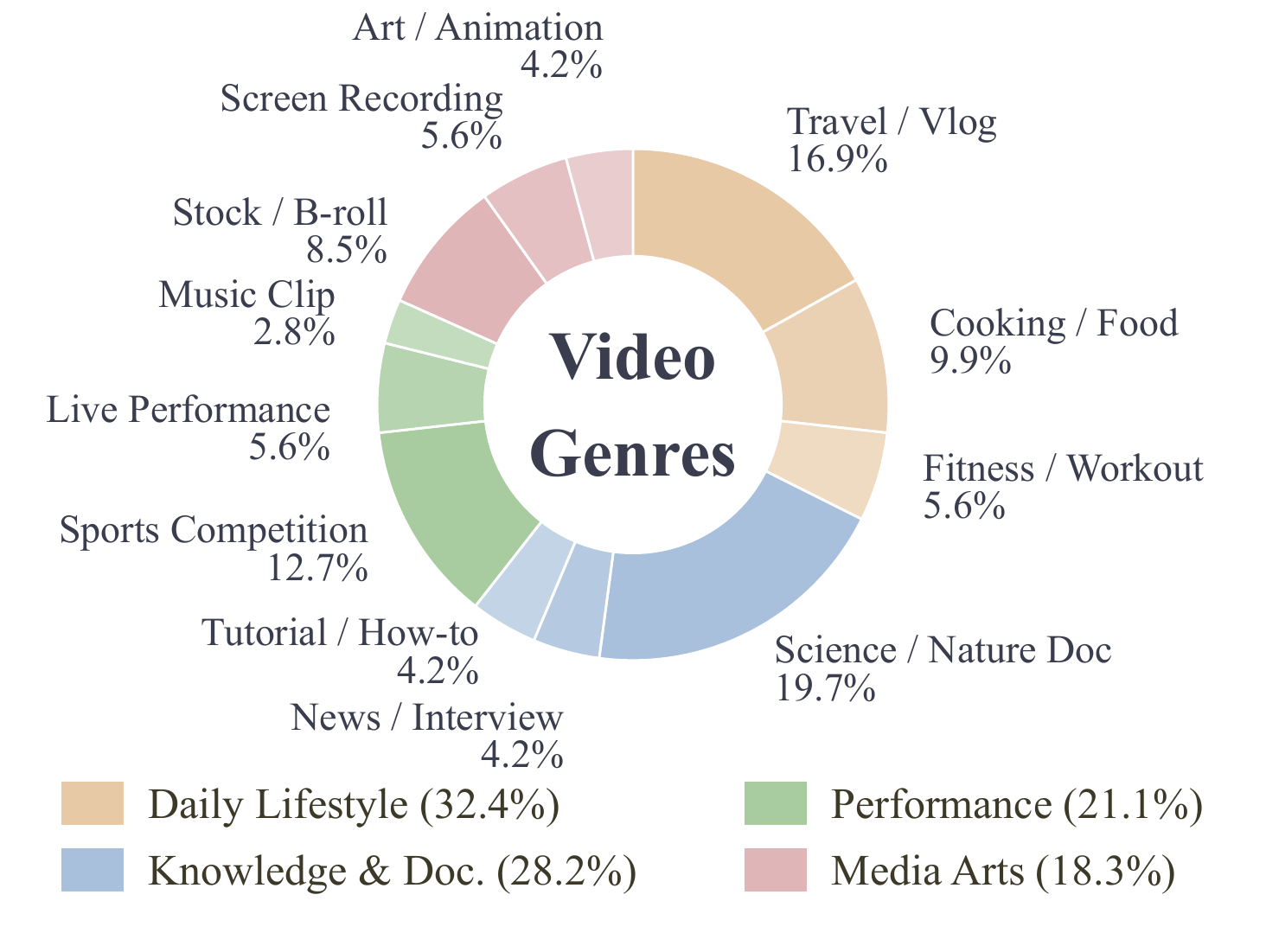}
    \caption{Video Genres.}
  \end{subfigure}
  \hfill
  \begin{subfigure}{0.48\textwidth}
    \centering
    \includegraphics[width=\linewidth]{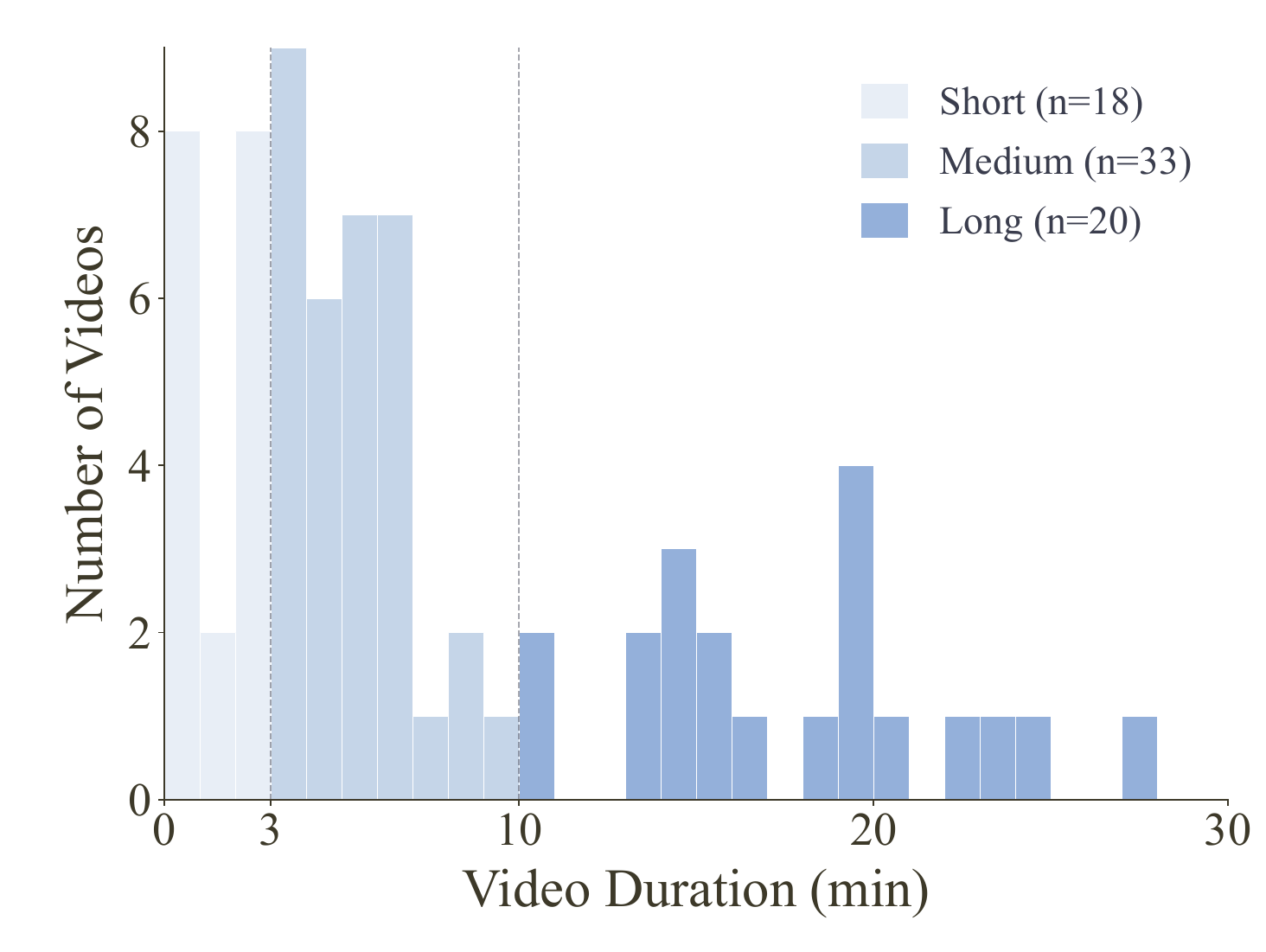}
    \caption{Video Durations.}
  \end{subfigure}

  \vspace{0.5em}

  \begin{subfigure}{0.48\textwidth}
    \centering
    \includegraphics[width=\linewidth]{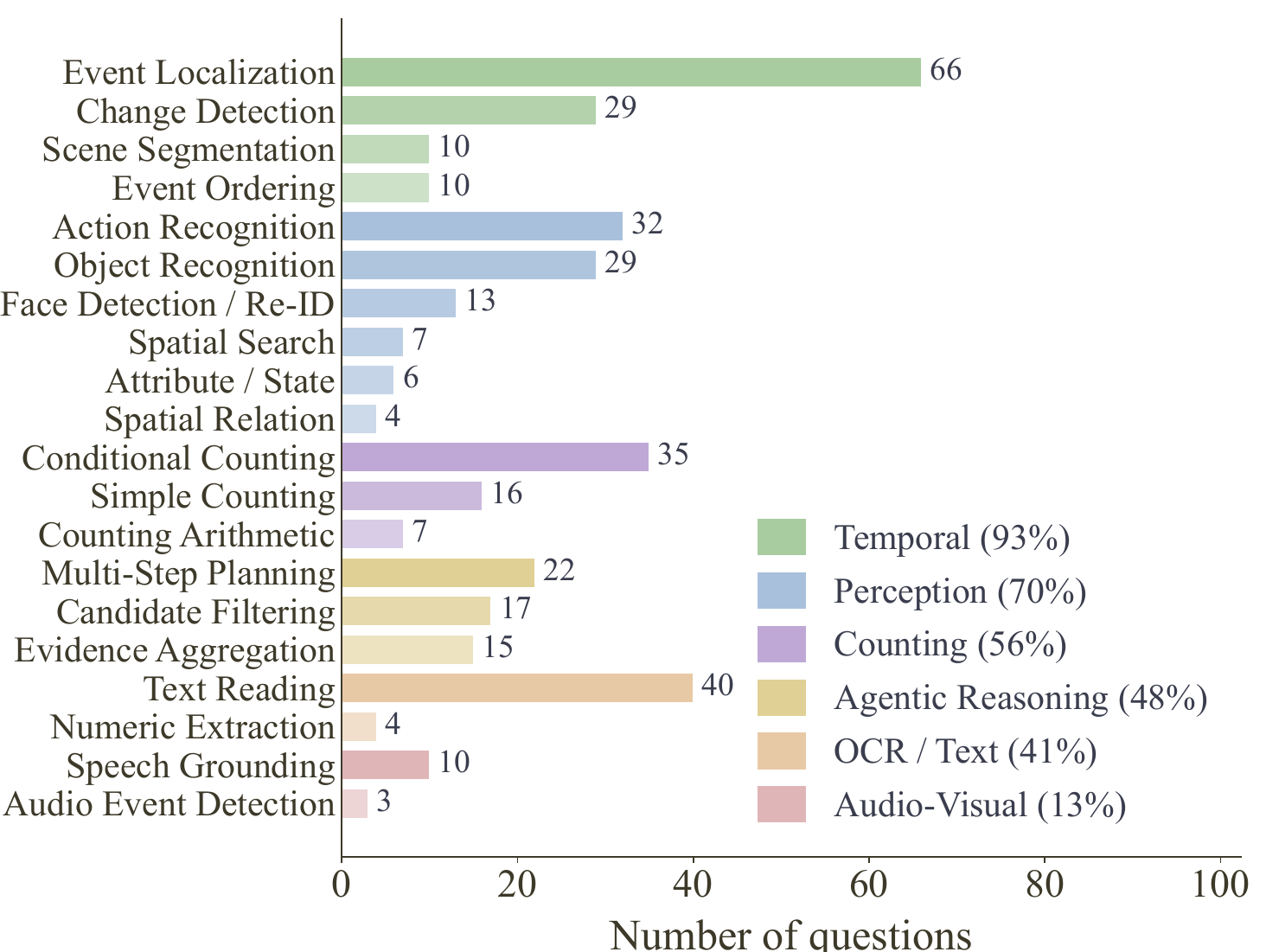}
    \caption{Required Skills.}
  \end{subfigure}
  \hfill
  \begin{subfigure}{0.48\textwidth}
    \centering
    \includegraphics[width=\linewidth]{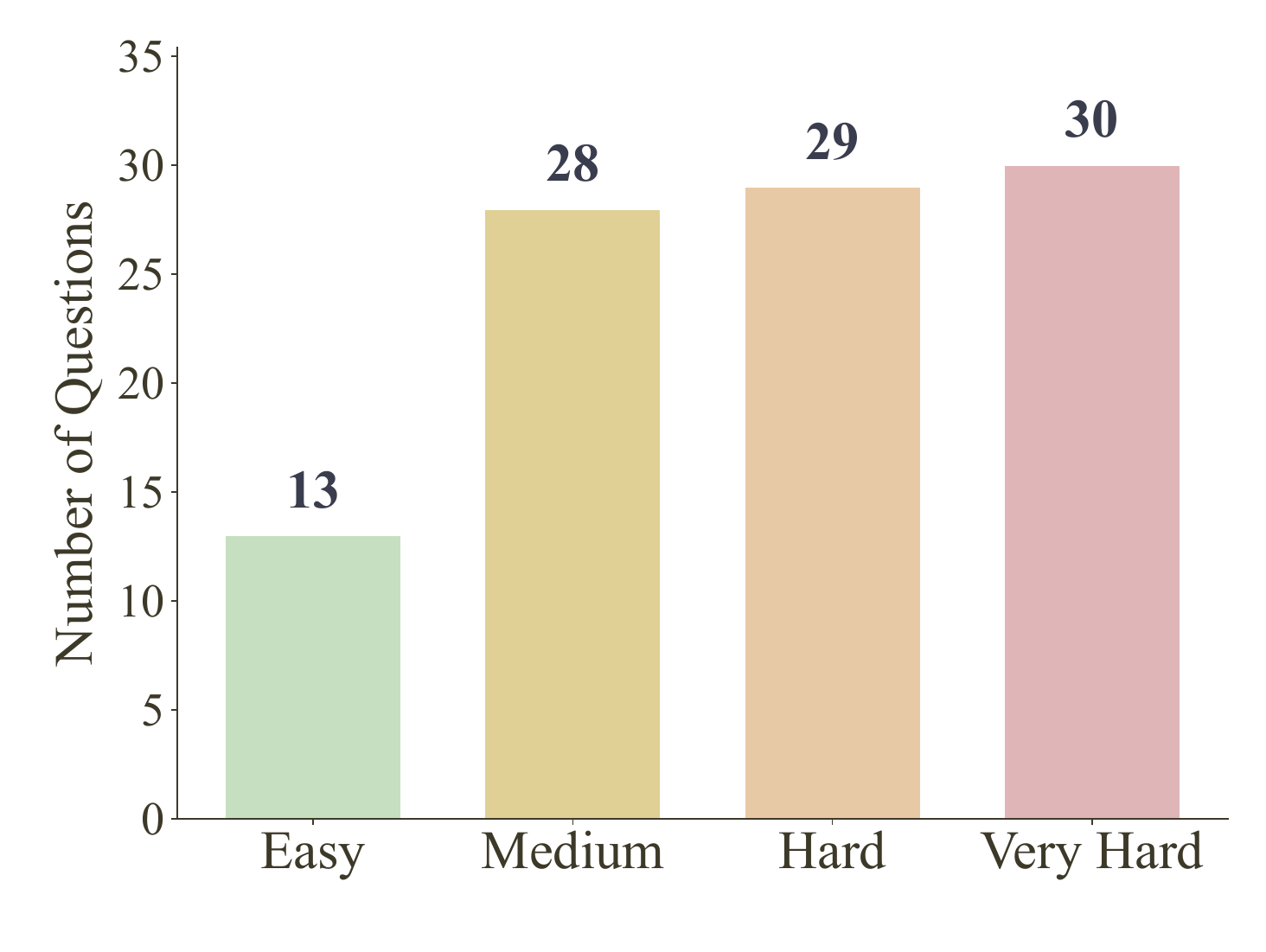}
    \caption{Question Difficulty.}
  \end{subfigure}

  \vspace{0.5em}
  \caption{
    AgentVidBench at a glance.
    Panels (a) and (b) summarize video-level statistics, including genre composition and duration distribution. Panels (c) and (d) summarize question-level statistics, including required skills and difficulty.
    Questions can have multiple skill tags, and the colors in (c) indicate the corresponding top-level capability groups.
    }
    \vspace{-0.4cm}
    \label{fig:dataset-statistics}
\end{figure}

We present the detailed statistics of AgentVidBench in Figure~\ref{fig:dataset-statistics}. At the video level, genres are hierarchically organized into four super-categories: \emph{Daily Lifestyle} (e.g., travel, cooking, fitness), \emph{Knowledge \& Documentary} (e.g., documentaries, news, tutorials), \emph{Performance} (e.g., sports, live performances), and \emph{Media Arts} (e.g., screen recordings, animations). Video durations span from 8 seconds to 27 minutes, with a median of approximately 5.5 minutes. The dataset encompasses short ($<3$ minutes), medium ($3$--$10$ minutes), and long ($>10$ minutes) clips, providing a diverse temporal distribution to evaluate both local video perception and long-context reasoning.

In addition, questions are multi-tagged with 20 fine-grained sub-skills across six top-level capabilities: \emph{temporal}, \emph{perception}, \emph{counting}, \emph{agentic reasoning}, \emph{OCR/text}, and \emph{audio--visual}. Demonstrating high compositionality, each query necessitates an average of 3.75 sub-skills (e.g., temporal localization coupled with multi-step planning). Furthermore, the 100 questions lean heavily toward challenging scenarios, distributed into 13 \texttt{easy}, 28 \texttt{medium}, 29 \texttt{hard}, and 30 \texttt{very hard} cases. 
We establish these difficulty tiers by synthesizing initial human judgments with the empirically observed performance of existing MLLMs and MLLM agents.
The statistical reliability of evaluation on AgentVidBench is analyzed in
Appendix~\ref{app:stat-reliability}.

\vspace{-0.1cm}
\subsection{Evaluation Pipeline}
\label{sec:evaluation-pipeline}

The evaluation pipeline of AgentVidBench consists of two stages: final-answer evaluation and trajectory evaluation. First, we measure standard final-answer accuracy by verifying whether the agent's ultimate prediction matches the ground truth. Next, we perform a milestone-based trajectory evaluation that compares each predicted trajectory against tool-agnostic evidence milestones. Drawing intuition from the process-level evaluation principles of MINERVA\,\cite{DBLP:conf/iccv/NagraniMIBMJHZVSSW25}, an LLM judge evaluates the predicted trajectory across five dimensions: Task Understanding (TU), Evidence Coverage (ECov), Evidence Grounding (EG), Evidence Completeness (ECom), and Reasoning Faithfulness (RF). We define the five dimensions by analyzing recurring failure modes observed in MLLM and MLLM agent trajectories. These five dimensions assess whether the agent understands the task constraints, acquires the necessary evidence, grounds its observations in the video, checks the completeness of the relevant temporal scope or candidate set when needed, and derives the answer consistently from the acquired evidence, respectively. 
We aggregate the five dimension scores into a unified and normalized trajectory score. This milestone-based approach simultaneously provides fine-grained diagnostics for agent failures and ensures fair, path-independent comparisons across diverse agentic frameworks. We provide more details on the evaluation pipeline and its validation in Appendix~\ref{app:evaluation-pipeline}.

\section{Experiments}

\vspace{-0.2cm}
\subsection{Settings}

\vspace{-0.1cm}
\paragraph{Models.}
To comprehensively evaluate diverse model families, we evaluate a broad suite of state-of-the-art (SOTA) MLLMs on our benchmark.
The evaluated models are grouped into proprietary and open-source MLLMs. 
For proprietary MLLMs, we include GPT-5\,\citep{DBLP:journals/corr/abs-2601-03267}, Claude Opus 4.7\,\citep{anthropic2026claudeopus47}, Gemini 2.5 Flash\,\citep{DBLP:journals/corr/abs-2507-06261}, and Gemini 2.5 Pro\,\citep{DBLP:journals/corr/abs-2507-06261}. 
For open-source MLLMs, we include Qwen3-VL-(4B, 8B)\,\citep{DBLP:journals/corr/abs-2511-21631}, Qwen3.5-(9B, 27B)\,\citep{qwen3.5}, and Gemma 4-(E2B, E4B, 26B-A4B, 31B)\,\citep{deepmind2026gemma4}.
In the single-turn MLLM evaluation setting, models process the video and generate an answer in a single pass, without iterative interaction or tool use. 
Following prior work~\citep{DBLP:journals/corr/abs-2604-05015}, we use 50 frames for single-turn evaluation of GPT-5 and Claude Opus 4.7, and adopt 1 fps sampling in all other settings (see more details in Appendix~\ref{app:video-input}).
In the MLLM agent evaluation setting~\citep{fan2025toolaugmented, DBLP:conf/nips/ZhangJGLLLL25, wang2025activevideoperceptioniterative}, the models serve as the core reasoning modules within various agentic video understanding frameworks, coordinating multi-step reasoning processes and determining when and how each framework should invoke external tools.

\begin{table*}[t!]
\centering
\caption{Performance comparison of proprietary and open-source MLLMs across
different agentic workflows on the AgentVidBench dataset. We report both Accuracy
(Acc.) and Trajectory (Traj.) scores for each configuration. Colored arrows next
to Video-ReAct (Ours) indicate absolute changes relative to the corresponding Single-Turn setting for the same model and metric.}
\small
\setlength{\tabcolsep}{5.65pt}
\renewcommand{\arraystretch}{1.12}
\begin{tabular}{@{}l cc cc cc cc cc@{}}
\toprule
\multirow{2}{*}{\textbf{Models}}
& \multicolumn{2}{c}{\textbf{Single-Turn}}
& \multicolumn{2}{c}{\textbf{STAR}~\cite{fan2025toolaugmented}}
& \multicolumn{2}{c}{\textbf{DVD}~\cite{DBLP:conf/nips/ZhangJGLLLL25}}
& \multicolumn{2}{c}{\textbf{AVP}~\cite{wang2025activevideoperceptioniterative}}
& \multicolumn{2}{@{\hspace{4pt}\kern\arrayrulewidth\hspace{4pt}}c}{\textbf{Video-ReAct (Ours)}} \\
\cmidrule(lr){2-3}\cmidrule(lr){4-5}\cmidrule(lr){6-7}\cmidrule(lr){8-9}\cmidrule(l){10-11}
& Acc. & Traj. & Acc. & Traj. & Acc. & Traj. & Acc. & Traj. & Acc. & Traj. \\
\midrule
\rowcolor{gray!18}
\multicolumn{11}{@{}l}{\textbf{Proprietary MLLMs}} \\
GPT-5              & 0.29 & 0.55 & 0.14 & 0.33 & 0.34 & 0.60 & 0.41 & 0.65 & \textbf{0.49} \textcolor{blue}{\scriptsize{$\uparrow$0.20}} & \textbf{0.76} \textcolor{blue}{\scriptsize{$\uparrow$0.21}} \\
Claude Opus 4.7    & 0.30 & 0.44 & 0.19 & 0.37 & 0.43 & 0.63 & 0.38 & 0.64 & \textbf{0.54} \textcolor{blue}{\scriptsize{$\uparrow$0.24}} & \textbf{0.79} \textcolor{blue}{\scriptsize{$\uparrow$0.35}} \\
Gemini 2.5 Flash   & \textbf{0.46} & 0.66 & 0.10 & 0.29 & 0.27 & 0.47 & \textbf{0.46} & 0.65 & 0.45 \textcolor{red}{\scriptsize{$\downarrow$0.01}} & \textbf{0.71} \textcolor{blue}{\scriptsize{$\uparrow$0.05}} \\
Gemini 2.5 Pro     & 0.51 & 0.71 & 0.13 & 0.35 & 0.30 & 0.47 & 0.56 & 0.71 & \textbf{0.58} \textcolor{blue}{\scriptsize{$\uparrow$0.07}} & \textbf{0.77} \textcolor{blue}{\scriptsize{$\uparrow$0.05}} \\
\midrule
\rowcolor{gray!18}
\multicolumn{11}{@{}l}{\textbf{Open-source MLLMs}} \\
\rowcolor{gray!7}
\multicolumn{11}{@{}l}{\quad\textit{Qwen family}} \\
Qwen3-VL-4B                & 0.16 & 0.34 & 0.07 & 0.22 & 0.08 & 0.24 & \textbf{0.38} & \textbf{0.68} & 0.37 \textcolor{blue}{\scriptsize{$\uparrow$0.21}} & 0.62 \textcolor{blue}{\scriptsize{$\uparrow$0.27}} \\
\quad\textit{+Think}       & 0.18 & 0.30 & 0.07 & 0.12 & 0.23 & 0.31 & \textbf{0.49} & \textbf{0.67} & 0.38 \textcolor{blue}{\scriptsize{$\uparrow$0.20}} & 0.65 \textcolor{blue}{\scriptsize{$\uparrow$0.36}} \\
Qwen3-VL-8B                & 0.18 & 0.41 & 0.04 & 0.21 & 0.10 & 0.22 & 0.38 & 0.63 & \textbf{0.40} \textcolor{blue}{\scriptsize{$\uparrow$0.22}} & \textbf{0.68} \textcolor{blue}{\scriptsize{$\uparrow$0.27}} \\
\quad\textit{+Think}       & 0.18 & 0.30 & 0.08 & 0.13 & 0.16 & 0.26 & \textbf{0.41} & 0.64 & 0.40 \textcolor{blue}{\scriptsize{$\uparrow$0.22}} & \textbf{0.70} \textcolor{blue}{\scriptsize{$\uparrow$0.41}} \\
Qwen3.5-9B                 & 0.19 & 0.37 & 0.08 & 0.20 & 0.32 & 0.50 & 0.34 & 0.66 & \textbf{0.40} \textcolor{blue}{\scriptsize{$\uparrow$0.21}} & \textbf{0.71} \textcolor{blue}{\scriptsize{$\uparrow$0.34}} \\
\quad\textit{+Think}       & 0.22 & 0.50 & 0.06 & 0.16 & 0.33 & 0.51 & \textbf{0.47} & 0.69 & 0.41 \textcolor{blue}{\scriptsize{$\uparrow$0.19}} & \textbf{0.72} \textcolor{blue}{\scriptsize{$\uparrow$0.22}} \\
Qwen3.5-27B                & 0.22 & 0.53 & 0.09 & 0.26 & 0.37 & 0.58 & 0.44 & 0.70 & \textbf{0.49} \textcolor{blue}{\scriptsize{$\uparrow$0.27}} & \textbf{0.78} \textcolor{blue}{\scriptsize{$\uparrow$0.25}} \\
\quad\textit{+Think}       & 0.25 & 0.57 & 0.09 & 0.27 & 0.38 & 0.53 & 0.46 & 0.71 & \textbf{0.50} \textcolor{blue}{\scriptsize{$\uparrow$0.25}} & \textbf{0.77} \textcolor{blue}{\scriptsize{$\uparrow$0.20}} \\
\midrule
\addlinespace[2pt]
\rowcolor{gray!7}
\multicolumn{11}{@{}l}{\quad\textit{Gemma family}} \\
Gemma-4-E2B                & 0.11 & 0.25 & 0.07 & 0.18 & 0.22 & 0.36 & 0.38 & \textbf{0.69} & \textbf{0.40} \textcolor{blue}{\scriptsize{$\uparrow$0.29}} & 0.67 \textcolor{blue}{\scriptsize{$\uparrow$0.42}} \\
\quad\textit{+Think}       & 0.07 & 0.25 & 0.08 & 0.19 & 0.30 & 0.28 & \textbf{0.53} & \textbf{0.71} & 0.38 \textcolor{blue}{\scriptsize{$\uparrow$0.31}} & 0.68 \textcolor{blue}{\scriptsize{$\uparrow$0.43}} \\
Gemma-4-E4B                & 0.16 & 0.35 & 0.05 & 0.22 & 0.33 & 0.45 & 0.44 & 0.68 & \textbf{0.51} \textcolor{blue}{\scriptsize{$\uparrow$0.35}} & \textbf{0.74} \textcolor{blue}{\scriptsize{$\uparrow$0.39}} \\
\quad\textit{+Think}       & 0.14 & 0.39 & 0.07 & 0.22 & 0.24 & 0.44 & \textbf{0.46} & 0.67 & 0.40 \textcolor{blue}{\scriptsize{$\uparrow$0.26}} & \textbf{0.67} \textcolor{blue}{\scriptsize{$\uparrow$0.28}} \\
Gemma-4-26B-A4B            & 0.21 & 0.44 & 0.09 & 0.25 & 0.31 & 0.55 & \textbf{0.51} & 0.71 & 0.49 \textcolor{blue}{\scriptsize{$\uparrow$0.28}} & \textbf{0.73} \textcolor{blue}{\scriptsize{$\uparrow$0.29}} \\
\quad\textit{+Think}       & 0.22 & 0.43 & 0.10 & 0.24 & 0.37 & 0.56 & 0.52 & 0.72 & \textbf{0.53} \textcolor{blue}{\scriptsize{$\uparrow$0.31}} & \textbf{0.76} \textcolor{blue}{\scriptsize{$\uparrow$0.33}} \\
Gemma-4-31B                & 0.38 & 0.59 & 0.08 & 0.28 & 0.31 & 0.59 & 0.48 & 0.70 & \textbf{0.51} \textcolor{blue}{\scriptsize{$\uparrow$0.13}} & \textbf{0.75} \textcolor{blue}{\scriptsize{$\uparrow$0.16}} \\
\quad\textit{+Think}       & 0.37 & 0.61 & 0.10 & 0.29 & 0.40 & 0.58 & \textbf{0.49} & 0.71 & 0.48 \textcolor{blue}{\scriptsize{$\uparrow$0.11}} & \textbf{0.75} \textcolor{blue}{\scriptsize{$\uparrow$0.14}} \\
\bottomrule
\end{tabular}
\vspace{-15pt}
\label{tab:main_results}
\end{table*}

\begin{figure}[t]
  \centering
  \includegraphics[width=\textwidth, trim=0 0.2cm 0 0]{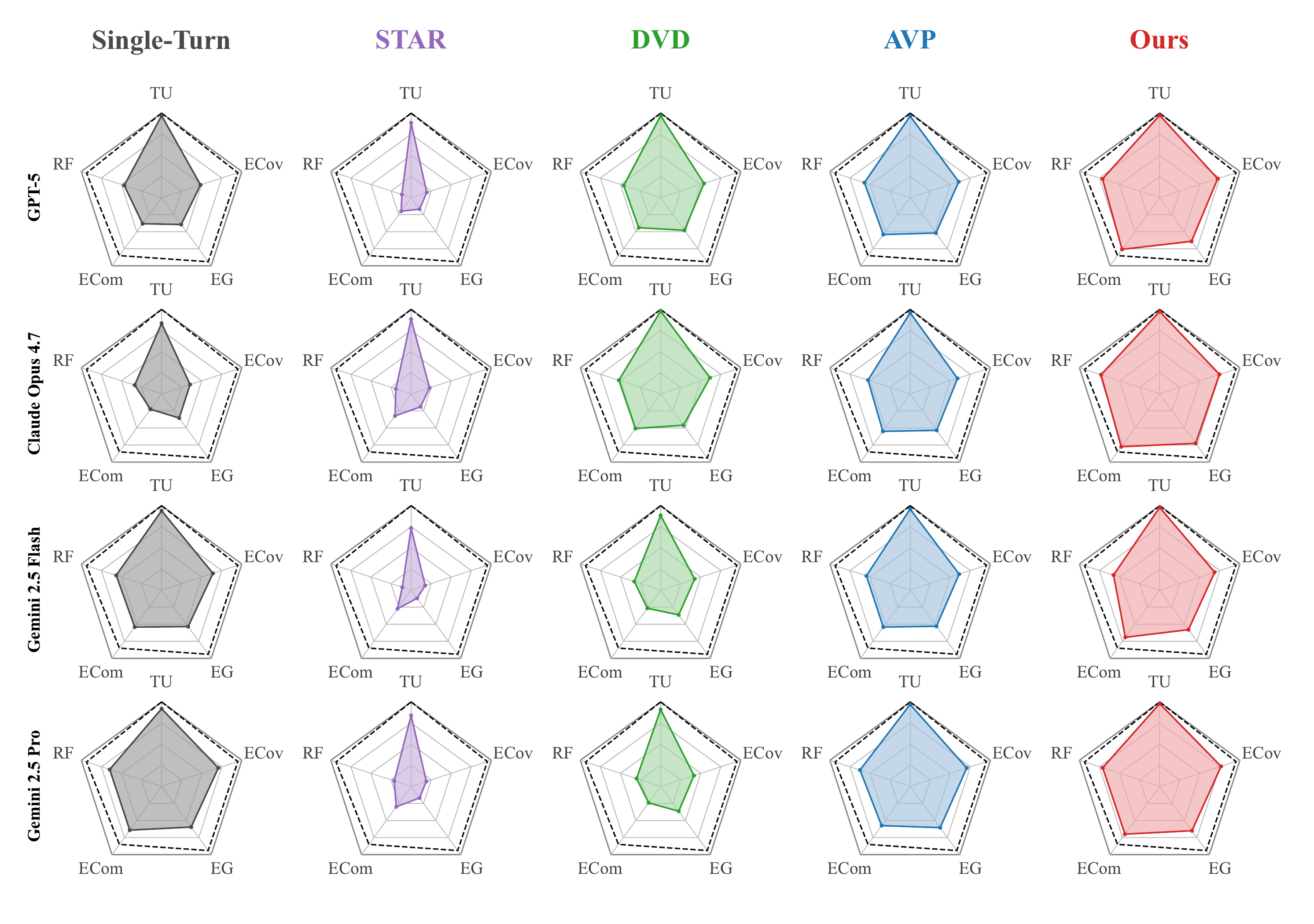}
  \vspace{-0.7cm}
  \caption{
  Trajectory score profiles of vanilla MLLMs and agentic frameworks. Each radar plot shows the average scores across five trajectory evaluation dimensions: Task Understanding (TU), Evidence Coverage (ECov), Evidence Grounding (EG), Evidence Completeness (ECom), and Reasoning Faithfulness (RF). The dashed outline indicates the human reference score, where larger filled regions represent stronger trajectory quality.
  }
  \vspace{-7pt}
  \label{fig:traj_results}
\end{figure}

\vspace{-0.4cm}
\paragraph{Agentic Frameworks.}
To evaluate our benchmark across diverse agentic design patterns, we consider frameworks that organize video evidence and guide visual perception in different ways. 
Throughout the evaluation, we replace only the main agent while keeping the original tools and execution pipeline of each method unchanged, thereby isolating how agent-level planning shapes video inspection under the same setting. 
STAR~\cite{fan2025toolaugmented} represents structured scheduling, alternating temporal and spatial tool use to narrow question-relevant regions. 
DVD~\cite{DBLP:conf/nips/ZhangJGLLLL25} follows search-centric planning over multi-granular evidence, combining global browsing, clip retrieval, and frame inspection. 
AVP~\cite{wang2025activevideoperceptioniterative} adopts active-perception planning, deciding what, where, and how to observe with reflection-based refinement. We also include our proposed simple yet effective task-adaptive reference agent, Video-ReAct, a lightweight ReAct-style baseline~\cite{DBLP:conf/iclr/YaoZYDSN023} designed for efficient evidence-seeking via controllable video inspection and transcript access.
Video-ReAct (Ours) operates as a single-agent ReAct loop in which the agent iteratively gathers evidence and judges its sufficiency using video observation tools, before finally producing the answer (see more details in Appendix~\ref{app:ours}).
Together, this diverse suite of baselines tests whether the benchmark remains effective across structured scheduling, search-centric retrieval, active observation, and task-adaptive planning.

\vspace{-0.2cm}
\paragraph{Metrics.}
We evaluate the performance of both MLLMs and MLLM agents primarily using standard accuracy for multiple-choice questions, which directly measures the correctness of the final outcome. Since a correct final answer does not inherently guarantee a sound reasoning process, we also evaluate trajectory scores for agentic workflows. This additional metric ensures that the final deduction is strictly grounded in a valid evidence-gathering process rather than mere chance. 

\subsection{Performance of Vanilla MLLMs and MLLM Agents}
We compare the performance of vanilla MLLMs and MLLM agents on AgentVidBench in terms of accuracy and trajectory score in Table~\ref{tab:main_results}.
We also present detailed trajectory scores across five dimensions as shown in Figure~\ref{fig:traj_results} and Appendix~\ref{app:detailed-trajectory-profiles}.
First, the results of different open-source MLLM agents under the same agentic framework show that performance varies depending on the main agent, even with a fixed set of tools.
Considering that the main agent does not directly observe the video and relies entirely on external tools for visual interpretation, these performance differences suggest that planning ability can substantially influence video understanding even under the same perceptual capacity.
However, as in the case of STAR, when the recognition capability of the external tools is insufficient, the agentic system can even underperform the single-turn performance of the vanilla MLLM itself.
We further compare different frameworks and observe that performance tends to improve as the main agent is given greater flexibility in deciding what to observe and where to focus.
For example, although both AVP and our method employ Gemini 2.5 Pro as a video observation tool, our method achieves better performance by allowing the agent to actively select the observation target and location, rather than following the fixed plan--observe--reflect pipeline adopted by AVP. We provide more details in Appendix~\ref{app:tool-budget}.
Overall, these results suggest that strong performance on AgentVidBench requires not only capable visual tools but also effective planning and adaptive observation strategies, making it a suitable benchmark for evaluating agentic frameworks in complex video understanding. We provide additional analyses of milestone coverage rates and dominant failure modes in Appendices~\ref{app:milestone-coverage-rate} and \ref{app:failure-modes}, respectively.

\subsection{Ablation Study on Adaptive Video Inspection}

\begin{wraptable}{r}{0.585\linewidth}
\vspace{-13pt}
\centering
\caption{Ablation study on adaptive video-inspection controls using Gemini 2.5 Pro.}
\vspace{-5pt}
\small
\begin{tabular}{@{}l ccc cc cc@{}}
\toprule
& \multicolumn{3}{c}{\textbf{Adaptive Controls}}
& \multicolumn{2}{c}{\textbf{AVP}}
& \multicolumn{2}{c}{\textbf{Ours}} \\
\cmidrule(lr){2-4}
\cmidrule(lr){5-6}
\cmidrule(lr){7-8}
\textbf{Settings} & FPS & Res. & Int. & Acc. & Traj. & Acc. & Traj. \\
\midrule
W/o FPS      & 1 & \checkmark & \checkmark & 0.50 & 0.69 & 0.50 & 0.73 \\
W/o Res.     & \checkmark & low & \checkmark & 0.49 & 0.68 & 0.56 & 0.73 \\
W/o Res.     & \checkmark & med & \checkmark & 0.43 & 0.66 & 0.55 & 0.77 \\
W/o Int.     & \checkmark & \checkmark & full & 0.46 & 0.68 & 0.50 & 0.72 \\
\midrule
Full         & \checkmark & \checkmark & \checkmark & 0.56 & 0.71 & 0.58 & 0.77 \\
\bottomrule
\end{tabular}
\label{tab:ablation-modules}
\vspace{-10pt}
\end{wraptable}

To better understand which types of video control matter for agentic video understanding within AgentVidBench, we ablate the main tool-control axes available to the agent when deciding how to inspect the video.
These axes correspond to complementary ways of inspecting video evidence: \emph{fps} controls temporal granularity (FPS), \emph{resolution} controls spatial fidelity (Res.), and \emph{time interval} controls which portion of the video is observed (Int.). 
All ablations use Gemini~2.5~Pro on the same evaluation set, and we report final-answer accuracy and trajectory score.

The results in Table~\ref{tab:ablation-modules} show that removing adaptive control generally reduces accuracy, while the impact of each control axis varies across agentic frameworks. Specifically, while the resolution and time interval tools play a dominant role in the AVP framework, all three tools demonstrate comparable importance in our framework. 
These results indicate that AgentVidBench is effective for evaluating the gains obtained from adaptive inspection compared to passive or uniformly allocated observation.

\subsection{Efficiency Analysis}

\begin{wraptable}{r}{0.49\linewidth}
\vspace{-14.5pt}
\caption{
Accuracy vs. token efficiency. Eff. denotes token efficiency, computed as accuracy (Acc.) divided by million tokens (MTok.).
}
\vspace{-5pt}
\centering
\small
\setlength{\tabcolsep}{4pt}
\begin{tabular*}{\linewidth}{@{\extracolsep{\fill}}lccc}
\toprule
\textbf{Methods} & \textbf{Acc.} & \textbf{MTok.} & \textbf{Eff. (\%)} \\
\midrule
Single-Turn & 0.51 & 15.16 & 3.36 \\
AVP         & 0.56 & 10.60 & 5.28 \\
Ours        & 0.58 & 15.44 & 3.76 \\
\bottomrule
\end{tabular*}
\label{tab:token_efficiency}
\vspace{-10pt}
\end{wraptable}

We report performance per token in Table~\ref{tab:token_efficiency}, defining token efficiency as accuracy divided by the total number of tokens measured in millions.
Overall, single-turn baseline shows the lowest efficiency, achieving a score of 3.36\%.
This is because it processes the entire video regardless of which segments are relevant for solving the problem.
In contrast, agentic methods can use tools to selectively inspect the relevant parts of the video, leading to larger performance gains per token.
AVP achieves the highest token efficiency of 5.28\%, whereas our method achieves the best accuracy of 58\% while remaining more efficient than the single-turn baseline.
This result indicates that tool-based reasoning improves token utilization, demonstrating its effectiveness from an efficiency perspective as well.

\vspace{-0.1cm}
\section{Conclusion}
In this work, we introduce AgentVidBench, a multi-hop video question answering benchmark designed to evaluate the spatial, temporal, and causal reasoning abilities of MLLM agents.
By combining challenging compositional questions with step-by-step evidence milestones, our benchmark enables evaluation beyond final-answer accuracy and supports fine-grained analysis of agent reasoning trajectories.
Through extensive experiments, we show that existing single-turn MLLMs still struggle with multi-hop video understanding, while effective agentic workflows can improve both answer accuracy and evidence-grounded reasoning quality.
Given the growing interest in MLLM agents and multimodal video tasks, we believe AgentVidBench can serve as a valuable benchmark for future research on multimodal agent reasoning. Furthermore, exploring methods that automatically learn optimal harnesses or weights\,\cite{DBLP:journals/corr/abs-2603-28052, kim2026whale}, as well as training small-scale MLLMs via reinforcement learning, remains a promising direction for future work.

\paragraph{Limitations.}
Although AgentVidBench enables rigorous evaluation through carefully curated questions paired with fine-grained trajectories, its current coverage remains limited. 
This limitation primarily stems from two factors: the strict requirement to use openly licensed videos and the labor-intensive nature of designing questions that require multi-step reasoning and constructing validated trajectories.
To address this limitation, future work will explore automated data collection methods for scaling AgentVidBench without compromising annotation quality.

\bibliographystyle{abbrv}
\bibliography{neurips_2026}

\clearpage
\newpage
\appendix

\section{Appendix -- Comparison of AgentVidBench with Existing Benchmarks}
\label{sec:comparison_agentvidbench}

We note that AgentVidBench uniquely evaluates active, tool-aware agent trajectories, distinguishing it from existing datasets. Although VideoMME\,\cite{DBLP:conf/cvpr/FuDLLRZWZSZCLLZ25}, LongVideoBench\,\cite{DBLP:conf/nips/WuLCL24}, and MINERVA\,\cite{DBLP:conf/iccv/NagraniMIBMJHZVSSW25} are all Video QA datasets, VideoMME and LongVideoBench are limited to evaluating final-answer accuracy, while MINERVA is restricted to the static reasoning of passive, single-pass MLLMs. In contrast, AgentVidBench specifically measures an agent's ability to dynamically invoke tools and explicitly gather evidence using tool-agnostic evidence milestones. A comprehensive comparison summary is provided in Table~\ref{tab:benchmarks_pro}.
\vspace{-0.2cm}
\begin{table}[htbp]
\centering
\caption{Comparison of video benchmarks.}
\label{tab:benchmarks_pro}
\begin{tabular}{lccc}
\toprule
\textbf{Benchmarks} & \textbf{Video QA} & \textbf{Reasoning Trajectory} & \textbf{Tool-aware Agent Trajectory} \\ 
\midrule
VideoMME       & \checkmark & $\times$   & $\times$   \\
LongVideoBench & \checkmark & $\times$   & $\times$   \\
MINERVA        & \checkmark & \checkmark & $\times$   \\
\midrule
AgentVidBench  & \checkmark & \checkmark & \checkmark \\ 
\bottomrule
\end{tabular}
\end{table}

Our controlled experiments demonstrate that MINERVA's evaluation metric (MiRA) is unsuitable for diverse agentic workflows, as it unfairly penalizes valid reasoning trajectories simply for utilizing different tool log formats (e.g., frame indices instead of MM:SS timestamps). In contrast, our tool-agnostic milestone judge robustly evaluates the actual evidence acquired regardless of the specific tool interface, while still strictly penalizing genuine evidence omission. This empirical robustness proves that AgentVidBench provides a fairer and more accurate evaluation protocol for assessing MLLM agents.
\section{Appendix -- Theoretical Details}
\label{sec:theoretical_details}

\subsection{Formal Fixed-Depth Consequence}
\label{sec:fixed_depth_consequence}

The informal corollary in the main text states that a fixed-depth, single-pass
transformer cannot solve arbitrarily deep worst-case multi-hop reasoning tasks
in a single forward pass. We formalize this statement by directly applying Theorem~5.1 of
Yao et al.~\citep{DBLP:conf/emnlp/YaoDZHK25}, using their \(k\)-hop relation-composition setting and assumptions.

\begin{corollary}[Formal version of Corollary~\ref{cor:informal_fixed_depth}]
\label{cor:fixed_depth_khop}
Consider any fixed causal transformer architecture with \(L\) layers, \(H\)
attention heads, hidden dimension \(d\), and \(p\)-bit precision. Under the
query-independent-attention assumption in
Theorem~\ref{thm:yao_khop_lower_bound}, there exists a sufficiently large hop
count \(k\), a sufficiently large entity universe \(E\), and a relation set
\(\mathcal{R}\subseteq E^E\) such that this transformer cannot solve the
corresponding worst-case \(k\)-hop relation-composition task in a single forward
pass.
\end{corollary}

This corollary follows immediately from the following known lower bound.

\subsection{Known \(k\)-Hop Depth Lower Bound}
\label{sec:known_khop_lower_bound}

The theorem below is Theorem~5.1 of Yao et al.~\citep{DBLP:conf/emnlp/YaoDZHK25}, stated here for completeness.
It formalizes a worst-case lower bound for implicit \(k\)-hop relation
composition in a single transformer forward pass.

\begin{theorem}[Worst-case \(k\)-hop lower bound]
\label{thm:yao_khop_lower_bound}
Let \(E\) be an entity universe and let \(k \le |E|-2\). Then there exists a relation
set \(\mathcal{R}\subseteq E^E\) such that the following holds.

Consider a causal transformer operating with \(p\)-bit precision, hidden
dimension \(d\), \(H\) attention heads, and \(L\) layers. Suppose the model
solves the \(k\)-hop relation-composition task over \(E\) and
\(\mathcal{R}\), where the target answer is
\[
    r_k \circ r_{k-1} \circ \cdots \circ r_1(e).
\]
Assume that the attention patterns do not depend on the query \(e\). These attention patterns may
still depend on the relation sequence \(r_1,\ldots,r_k\). Then
\[
    L \ge \frac{k}{8pdH}.
\]
\end{theorem}

\begin{proof}[Reference]
This is Theorem~5.1 of Yao et al.~\citep{DBLP:conf/emnlp/YaoDZHK25}, restated
in our notation.
\end{proof}

\subsection{Proof of Corollary~\ref{cor:fixed_depth_khop}}
\label{sec:proof_fixed_depth_consequence}

\begin{proof}
Fix the given transformer architecture, and let \(L,p,d,H\) denote its depth,
precision, hidden dimension, and number of attention heads. Choose an integer
\(k\) such that
\[
    k > 8pdHL.
\]
Then choose an entity universe \(E\) large enough that
\[
    k \le |E|-2.
\]
By Theorem~\ref{thm:yao_khop_lower_bound}, there exists a relation set
\(\mathcal{R}\subseteq E^E\) such that solving the corresponding \(k\)-hop
relation-composition task requires
\[
    L \ge \frac{k}{8pdH}.
\]
However, by our choice of \(k\),
\[
    \frac{k}{8pdH} > L.
\]
Thus the fixed \(L\)-layer transformer does not have sufficient depth to solve
this worst-case \(k\)-hop relation-composition task in a single forward pass
under the stated assumption. This proves the corollary.
\end{proof}

\subsection{Interpretation and Implication for Multi-Hop Video QA}
\label{sec:theory_interpretation}

Corollary~\ref{cor:fixed_depth_khop} is the formal consequence we need for our
benchmark motivation: for any fixed transformer architecture, there exists a
worst-case \(k\)-hop relation-composition task that cannot be solved in a single
forward pass once \(k\) is sufficiently large. The supporting lower bound can be
written as
\[
    LpdH = \Omega(k),
\]
where \(L\) is the number of transformer layers, while \(p,d,H\) capture
per-layer capacity through numerical precision, hidden dimension, and the number of
attention heads. Thus, the result should be read as a width--depth--precision
tradeoff: increasing per-layer capacity can partially compensate for depth, but for a
fixed architecture under single-pass inference---without externalized
intermediate states, recurrence, or other adaptive test-time computation---the
required depth grows with the number of composed relations.

The formal relation-composition task is not itself a theorem about video QA.
Rather, we use it as a theoretical abstraction of a broader computational
bottleneck: fixed-depth single-pass transformers have bounded internal
sequential computation, and sufficiently deep dependent reasoning can exceed
this budget. This perspective motivates our benchmark design.

In video QA, the analogous difficulty arises when answering a question requires
a sequence of dependent evidence updates. A model may need to acquire evidence
from one part of the video, use that evidence to guide the search for another
part, verify a relation between objects or events, and then update its answer
based on the new evidence. A single-pass VLM must compress this entire process
into one fixed-depth forward pass. Following this motivation, we construct AgentVidBench around this regime.

\section{Appendix -- AgentVidBench Construction Details}
\label{app:benchmark-construction}

We provide additional details for Section~\ref{sec:benchmark-construction}, including the concrete configuration of every stage used to build AgentVidBench.

\subsection{Video Collection}
\label{app:video-sources}

AgentVidBench contains 71 videos collected from publicly accessible or author-owned sources. Publicly sourced videos are primarily drawn from YouTube and Wikimedia Commons, with additional clips from NASA, DVIDS, Vimeo, Current Biology, and the US National Park Service. We also include four self-recorded videos. Table~\ref{tab:video-sources} summarizes the distribution of source platforms.

For each video, we manually checked the license or reuse terms before inclusion. All videos were verified to be covered by Creative Commons licenses, public-domain status, or author ownership permitting release with the benchmark package. Table~\ref{tab:video-licenses} summarizes the resulting license distribution.

\begin{table}[h!]
  \centering
  \caption{Distribution of video sources in AgentVidBench.}
  \begin{tabular}{lc}
    \toprule
    Source platform   & \# Videos \\
    \midrule  
    YouTube                  & 44 \\
    Wikimedia Commons        & 10 \\
    NASA                     & 8 \\
    Self Recorded            & 4 \\
    DVIDS (US DoD)           & 2 \\
    Current Biology          & 1 \\
    Vimeo                    & 1 \\
    US National Park Service  & 1 \\
    \midrule
    \textbf{Total}         & \textbf{71} \\
    \bottomrule
  \end{tabular}
  \label{tab:video-sources}
\end{table}

\begin{table}[h!]
  \centering
  \caption{Distribution of video licenses in AgentVidBench.}
  \begin{tabular}{lc}
    \toprule
    License   & \# Videos \\
    \midrule  
    YouTube CC BY (reuse allowed)     & 43 \\
    Public Domain                     & 14 \\
    CC BY-NC 4.0                      & 4 \\
    CC BY-SA 4.0                      & 4 \\
    CC BY 4.0                         & 3 \\
    CC BY-SA 3.0                      & 2 \\
    CC BY 3.0                         & 1 \\
    \midrule
    \textbf{Total}         & \textbf{71} \\
    \bottomrule
  \end{tabular}
  \label{tab:video-licenses}
\end{table}

\subsection{Annotation Process and Agreement Statistics}
\label{app:annotation-process}
All benchmark questions were constructed by the five authors, with each question created by a single author. The author prepared the question, its 26 answer options, the answer key, and the explanation. Two authors contributed exclusively to question construction and did not participate in the blind review; the remaining three authors served as reviewers. For each question, the author's answer key served as one answer, and two reviewers who did not author the question independently solved it without access to that key, so that each question was answered by three annotators. The three answers were then compared, and agreement on the 26-way answer choice was high, with a Fleiss’ $\kappa$ of $0.826$. The blind review flagged 22 of the 100 questions for reconciliation, primarily because a reviewer selected a different option or identified ambiguity in the wording. Subsequent review led to seven answer-key corrections, the replacement of one question, and wording clarifications for 23 questions, without changing their answers.

\subsection{Milestone Extraction Details}
\label{app:evidence-milestones}

Final-answer accuracy alone cannot determine whether an agent actually acquired the video evidence needed to justify its answer. To evaluate the agent's investigation process, annotators manually wrote a detailed ground-truth reasoning note for each question, describing how a human would investigate the video to reach the answer. We then abstracted these notes into structured evidence milestones, where each milestone specifies a tool-agnostic evidence or reasoning step needed to support the answer.

Milestones are intentionally more abstract than raw solution trajectories. Rather than prescribing a particular tool sequence, timestamp, frame index, or navigation path, they define the underlying evidence requirements of a question. This allows agents with different interfaces or search strategies to be judged against the same evidence reference.

\paragraph{Milestone Schema.}
Each milestone consists of a stable identifier, an evidence type, and a tool-agnostic natural-language description. The identifier is used to align milestone-level coverage judgments, the type indicates the kind of evidence or reasoning step required, and the description states what evidence should be acquired.

\paragraph{Tool-agnostic Annotation Principles.}
Milestones are written to specify \emph{what} evidence is needed, not \emph{how} the evidence should be obtained. This is necessary because agentic video frameworks vary in their tool interfaces and interaction granularity: some use low-level controls such as seeking and frame-rate adjustment, others use high-level video-analysis calls, and zero-shot baselines may receive the video in a single pass. Comparing raw tool traces would therefore privilege one interface or search strategy over another. 
To avoid encoding path-specific cues, milestone descriptions remove or abstract away narrow timestamps, exact time ranges, tool names, occurrence timestamp lists, and solution-specific decomposition details from the original human reasoning note. For example, a human explanation may mention that a clue appears at a particular timestamp, but the milestone should instead describe the evidence to be found at that moment. Similarly, for counting questions, milestones should not enumerate every timestamped occurrence; they should describe the counting target and any relevant exclusion condition. Table~\ref{tab:milestone-examples} gives representative examples of path-specific cues and their tool-agnostic milestone counterparts.

\begin{table}[h!]
    \centering
    \caption{Examples of path-specific cues and their tool-agnostic milestone counterparts.}
    \label{tab:milestone-examples}
    \small
    \setlength{\tabcolsep}{4.5pt}
    \renewcommand{\arraystretch}{1.15}
    \begin{tabularx}{\linewidth}{
        >{\raggedright\arraybackslash}p{0.17\linewidth}
        >{\raggedright\arraybackslash}X
        >{\raggedright\arraybackslash}X
    }
    \toprule
    Path-specific cue & Path-specific reference & Tool-agnostic milestone \\
    \midrule
    Tool name &
    Use the OCR tool to read the model code. &
    Read the model code on the bezel as \texttt{CF-19}. \\

    Specific timestamp &
    At 02:13, inspect the laptop screen. &
    Locate the close-up moment where the laptop screen and bezel text are clearly visible. \\

    Fixed time range &
    The relevant evidence is in the 00:20--00:50 segment. &
    Locate the driving segment where the relevant event is visible. \\

    Timestamp list &
    Check the moments at 2:49, 2:50, 3:36, and 3:51. &
    Identify all candidate moments showing the target action. \\

    Exact arithmetic chain &
    Sum the counts as \texttt{4+2+1+7+10+2+5+3+3}. &
    Sum the extracted counts to obtain the final total. \\
    \bottomrule
    \end{tabularx}
\end{table}

\paragraph{Milestone Statistics.}
Rather than prescribing full solution trajectories, the milestone representation provides a compact evidence checklist. Each question is represented by 3--7 milestones, with most questions containing four or five evidence units. Table~\ref{tab:milestone-counts} summarizes the number of milestones per question.
Table~\ref{tab:milestone-types} shows that temporal localization is the most frequent milestone type, followed by visual recognition, exclusion checking, counting, context identification, and OCR/text reading. This distribution reflects the benchmark's emphasis on locating relevant evidence over time and combining it with perceptual, textual, counting, or exclusion-based reasoning steps.

\begin{table}[h!]
    \centering
    \caption{Number of evidence milestones per question in AgentVidBench.}
    \label{tab:milestone-counts}
    \small
    \setlength{\tabcolsep}{8pt}
    \renewcommand{\arraystretch}{1.12}
    \begin{tabular}{lcc}
    \toprule
    \# Milestones & \# Questions & Ratio (\%) \\
    \midrule
    3 & 14 & 14.0 \\
    4 & 32 & 32.0 \\
    5 & 40 & 40.0 \\
    6 & 12 & 12.0 \\
    7 & 2 & 2.0 \\
    \bottomrule
    \end{tabular}
\end{table}

\begin{table}[h!]
    \centering
    \caption{Distribution of evidence milestone types. Percentages are computed over 456 milestones.}
    \label{tab:milestone-types}
    \small
    \setlength{\tabcolsep}{8pt}
    \renewcommand{\arraystretch}{1.12}
    \begin{tabular}{lcc}
    \toprule
    Milestone type & Occurrences & Ratio (\%) \\
    \midrule
    Temporal localization & 113 & 24.8 \\
    Visual recognition & 69 & 15.1 \\
    Exclusion check & 59 & 12.9 \\
    Counting & 54 & 11.8 \\
    Context identification & 43 & 9.4 \\
    OCR/text reading & 41 & 9.0 \\
    Verification & 29 & 6.4 \\
    Audio-visual evidence & 28 & 6.1 \\
    Comparison & 11 & 2.4 \\
    Arithmetic & 9 & 2.0 \\
    \bottomrule
    \end{tabular}
\end{table}

\paragraph{Reference Validation.}
We use milestone-only references as the default input to the trajectory judge. To validate this choice, we compare judge scores obtained with milestone-only references against scores obtained with the original human answer explanations. As shown in Table~\ref{tab:milestone-human-correlation}, the two references produce strongly correlated process-quality scores across the five axes, suggesting that milestones preserve the core evidence information needed for trajectory evaluation.
At the same time, milestone references substantially reduce path-specific constraints. Human explanations often contain timestamps, time ranges, timestamp lists, or solution-specific arithmetic decompositions. Directly using them as judge references can reward trajectories that follow the same human solution path and penalize different but valid search strategies. Table~\ref{tab:milestone-cues} shows that milestone references preserve the judging signal while removing most path-specific cues, making them a more suitable default for evaluating agents that may reach the same evidence through different valid paths.

\begin{table}[h!]
    \centering
    \caption{Correlation between human-reference and milestone-reference judge scores.}
    \label{tab:milestone-human-correlation}
    \small
    \setlength{\tabcolsep}{8pt}
    \renewcommand{\arraystretch}{1.12}
    \begin{tabular}{lc}
    \toprule
    Process-quality axis & Pearson $r$ \\
    \midrule
    Task understanding & 0.656 \\
    Ground-truth evidence coverage & 0.756 \\
    Evidence grounding & 0.814 \\
    Evidence completeness & 0.746 \\
    Reasoning faithfulness & 0.837 \\
    \midrule
    Mean & 0.762 \\
    \bottomrule
    \end{tabular}
\end{table}

\begin{table}[h!]
    \centering
    \caption{Path-specific cues detected in human explanations and milestone references.}
    \label{tab:milestone-cues}
    \small
    \setlength{\tabcolsep}{8pt}
    \renewcommand{\arraystretch}{1.12}
    \begin{tabular}{lcc}
    \toprule
    Reference input & Affected Qs & Tool-agnostic rate \\
    \midrule
    Human answer explanation & 85/100 & 15.0\% \\
    Evidence milestones & 1/100 & 99.0\% \\
    \bottomrule
    \end{tabular}
\end{table}

\subsection{Taxonomy Labels} 
\label{app:taxonomy}

We annotate each video and question with taxonomy metadata to support fine-grained analysis of benchmark composition and model performance. Taxonomy labels with deterministic criteria, such as \emph{duration}, are assigned based on predefined rules, whereas \emph{production style} and \emph{audio type} are determined via manual review. For \emph{genre} and \emph{required skills}, we utilize a human-in-the-loop approach where MLLM-generated initial labels are rigorously verified by human annotators. Finally, \emph{difficulty} is initially estimated by humans and then empirically calibrated using the observed performance of the evaluated MLLMs and agentic frameworks to mitigate subjective bias.

Table~\ref{tab:taxonomy-overview} summarizes the taxonomy axes, whether they are assigned at the video or question level, and their label cardinality. Video-level metadata describes the content genre, duration, production style, and audio characteristics of each source video, while question-level metadata describes the skills required to answer the question and its difficulty.

Video genre and required skills are organized hierarchically. Table~\ref{tab:video-genre-taxonomy} lists the four video-genre super-categories and their corresponding fine-grained genres. Production style, audio type, and required skills are multi-label axes, since a video or question may involve multiple styles, modalities, or skills. Their label sets are provided in Tables~\ref{tab:production-audio-taxonomy} and~\ref{tab:skill-taxonomy}.

\begin{table}[h!]
    \centering
    \caption{Overview of taxonomy axes in \textsc{AgentVidBench}.}
    \label{tab:taxonomy-overview}
    \small
    \setlength{\tabcolsep}{6pt}
    \renewcommand{\arraystretch}{1.12}
    \begin{tabular}{llcl}
    \toprule
    Level & Axis & \# Labels & Description \\
    \midrule
    Video & Genre & 12 & Content genre, grouped into 4 super-categories. \\
          & Duration & 3 & Short, medium, or long duration bin. \\
          & Production style & 14 & Camera, editing, and presentation style. \\
          & Audio type & 5 & Dominant audio characteristics. \\
    \midrule
    Question & Required skills & 20 & Multi-label skills, grouped into 6 capabilities. \\
             & Difficulty & 4 & Question difficulty label. \\
    \bottomrule
    \end{tabular}
\end{table}

\begin{table}[h!]
    \centering
    \caption{Video genre taxonomy.}
    \label{tab:video-genre-taxonomy}
    \small
    \setlength{\tabcolsep}{6pt}
    \renewcommand{\arraystretch}{1.2}
    \begin{tabular}{p{0.25\linewidth}>{\raggedright\arraybackslash}p{0.68\linewidth}}
    \toprule
    Super-category & Description \\
    \midrule
    Daily Lifestyle &
    Everyday activity-oriented videos. \newline
    \emph{Genres:} Travel/Lifestyle Vlog, Cooking/Food, Fitness/Workout. \\

    Knowledge \& Documentary &
    Informational, educational, or documentary-style videos. \newline
    \emph{Genres:} Science/Nature Documentary, News/Interview, Tutorial/How-to. \\

    Performance &
    Event- or performance-centered videos. \newline
    \emph{Genres:} Sports Competition, Live Performance, Music Clip. \\

    Media Arts &
    Screen-based, artistic, or contextual visual content. \newline
    \emph{Genres:} Screen Recording, Art/Animation, Stock/B-roll. \\
    \bottomrule
    \end{tabular}
\end{table}

\begin{table}[h!]
    \centering
    \caption{Production-style and audio-type labels. Labels are multi-label at the video level.}
    \label{tab:production-audio-taxonomy}
    \small
    \setlength{\tabcolsep}{6pt}
    \renewcommand{\arraystretch}{1.25}
    \begin{tabular}{p{0.20\linewidth}p{0.72\linewidth}}
    \toprule
    Taxonomy & Labels \\
    \midrule
    Production style &
    scene-cut edit, static camera, handheld camera, single continuous shot, drone/aerial, tracking camera, animation, screen capture, subtitle present, split screen, timelapse, text-overlay heavy, multicam switch, slow motion. \\
    \addlinespace[2pt]
    Audio type &
    speech, music, ambient sound, sound effects, silent. \\
    \bottomrule
    \end{tabular}
\end{table}

\begin{table}[h!]
    \centering
    \caption{Required-skill taxonomy. Labels are multi-label at the question level.}
    \label{tab:skill-taxonomy}
    \small
    \setlength{\tabcolsep}{5pt}
    \renewcommand{\arraystretch}{1.18}
    \begin{tabular}{p{0.22\linewidth}>{\raggedright\arraybackslash}p{0.30\linewidth}>{\raggedright\arraybackslash}p{0.38\linewidth}}
    \toprule
    Capability & Description & Fine-grained skills \\
    \midrule
    Temporal &
    Reasoning about time, sequence, or scene structure of the video. &
    event localization, event ordering, scene segmentation, change detection. \\

    Perception &
    Frame-level visual recognition of entities, attributes, or activities. &
    object recognition, action recognition, face detection / re-identification, attribute/state recognition, spatial relation, spatial search. \\

    Counting &
    Quantifying occurrences across the video, possibly under conditions or with arithmetic over multiple counts. &
    simple counting, conditional counting, counting arithmetic. \\

    Agentic reasoning &
    Higher-order skills: planning multi-step investigations, filtering candidate answers, or aggregating evidence across observations. &
    multi-step planning, candidate filtering, evidence aggregation. \\

    OCR/Text &
    Reading textual or numeric content rendered visually in the frame. &
    OCR/text reading, numeric value extraction. \\

    Audio-visual &
    Tasks that require integrating audio cues with visual content. &
    speech grounding, audio event detection. \\
    \bottomrule
    \end{tabular}
\end{table}
\section{Appendix -- Evaluation Pipeline Details}
\label{app:evaluation-pipeline}

As illustrated in Figure~\ref{fig:dataset-overview}(b), AgentVidBench evaluates each system output along two complementary branches. The first branch evaluates the final answer by matching the selected option against the ground-truth answer. This provides the standard multiple-choice accuracy metric. The second branch evaluates the reasoning trajectory produced before the final answer, using an LLM judge to assess whether the system acquired, grounded, and integrated the video evidence needed to justify its prediction.

The motivation for the second branch is that final-answer accuracy alone does not reveal how the answer was obtained. Two systems may select the same option while following very different processes: one may localize the relevant moments, verify distractors, and derive the answer from grounded observations, while another may rely on incomplete evidence or commit prematurely. We therefore evaluate trajectory-level process quality in addition to answer correctness.

For a given question, the trajectory judge receives the question text, answer options, a pre-classified sweep-applicability flag, pre-computed evidence milestones, and the rendered agent trajectory. The gold answer and externally parsed predicted answer letter are not provided to the judge, preventing answer label leakage. Final-answer correctness is computed separately by exact multiple-choice matching. The original human answer explanation is not provided; instead, the judge uses milestone-only references, as described in Appendix~\ref{app:evidence-milestones}. This design evaluates whether the trajectory recovers the underlying evidence needed for the answer, rather than whether it follows the same timestamp sequence or tool-use path as the human-written solution.

\subsection{Judge Input Construction}
Evaluated systems produce different kinds of intermediate outputs. Agentic frameworks may generate structured tool-use trajectories with calls, observations, and intermediate reasoning, whereas non-agentic baselines may produce only a direct textual response. Before judging, we render each prediction into a common text format containing the predicted final answer and all available intermediate observations, retrieved evidence, tool calls, and reasoning traces. This normalization allows the same judge prompt to be applied across heterogeneous systems without assuming a shared tool interface.

\subsection{Process-quality Axes}
We define the five process-quality axes by analyzing recurring failure modes observed in MLLM and MLLM-agent trajectories. These failures cluster around five stages of the investigation process: understanding the task, acquiring the required evidence, grounding observations in the video, checking whether the evidence search is complete, and faithfully deriving the answer from the acquired evidence. The trajectory judge therefore assigns scores along five axes: Task Understanding (TU), Ground-truth Evidence Coverage (ECov), Evidence Grounding (EG), Evidence Completeness (ECom), and Reasoning Faithfulness (RF). Each applicable axis is scored on a 0--2 scale, where 0 indicates a severe failure, 1 indicates partial success, and 2 indicates strong performance.

\begin{itemize}[leftmargin=1em]
    \item \textbf{Task Understanding (TU).} Measures whether the agent correctly understands the question, including the target entity, constraints, exclusion rules, temporal scope, and answer format. This axis captures failures such as solving for the wrong object, ignoring a qualifier, or misinterpreting ``first'', ``last'', ``only'', or similar constraints.

    \item \textbf{Ground-truth Evidence Coverage (ECov).} Measures whether the trajectory acquires the evidence specified by the reference milestones. This axis focuses on whether the necessary evidence is present in the trajectory, rather than whether the final option happens to be correct.

    \item \textbf{Evidence Grounding (EG).} Measures whether the observations claimed by the agent match the reference evidence. This includes OCR values, object identities, counts, visual attributes, spatial relations, audio cues, and other intermediate observations. An agent may search in the right place but still receive a low EG score if it misreads or misrecognizes the evidence.

    \item \textbf{Evidence Completeness (ECom).} Measures whether the agent sufficiently checks the relevant temporal scope or candidate set when the question requires evidence-completeness verification. This axis is especially important for counting, ordering, first/last, absence, or exhaustive-search questions. If a question is locally answerable from a single anchored observation, ECom is marked as not applicable and excluded from the applicable-axis mean.

    \item \textbf{Reasoning Faithfulness (RF).} Measures whether the final answer follows consistently from the evidence acquired in the trajectory. This axis captures errors such as unsupported leaps, arithmetic mistakes, contradictions between observations and conclusion, or incorrect option mapping.
\end{itemize}

\paragraph{Applicability of Evidence Completeness (ECom).}
Evidence Completeness (ECom) is applied only when a question requires evidence-completeness verification, such as checking the full temporal scope or multiple candidate instances before committing to an answer. Such questions include counting tasks, total or aggregate questions, first/last or Nth-instance questions, ordering/ranking questions, and cases where absence or exclusion must be verified across the video. For locally anchored questions, where a perfect agent could answer from a single well-localized observation, ECom is treated as not applicable. To ensure consistency across frameworks, we pre-determine ECom applicability for each question, include it as dataset metadata, and inject it into the judge input, rather than asking the judge to re-decide applicability independently for each trajectory.

\paragraph{Trajectory Score Aggregation.}
We aggregate the applicable process-axis scores into a normalized trajectory score:
\[
\mathrm{Traj}
=
\frac{1}{2}
\cdot
\frac{1}{|\mathcal{A}|}
\sum_{a \in \mathcal{A}} s_a,
\]
where \(s_a \in \{0,1,2\}\) is the score for axis \(a\), and \(\mathcal{A} \subseteq \{\mathrm{TU}, \mathrm{ECov}, \mathrm{EG}, \mathrm{ECom}, \mathrm{RF}\}\) is the set of applicable axes. When ECom is not applicable, it is excluded from \(\mathcal{A}\). We denote this normalized score as \emph{Traj.} throughout the experimental tables.

\subsection{Auxiliary Diagnostics} 
\label{app:auxiliary-diagnostics}
In addition to the process-axis scores used to compute \emph{Traj.}, the judge outputs auxiliary diagnostics for further analysis. First, it assigns a coverage label to each reference milestone: \texttt{covered}, \texttt{partial}, \texttt{incorrect}, or \texttt{missing}. These labels are not part of the main trajectory score; they are used to inspect whether each milestone is fully supported, partially addressed, contradicted by an incorrect observation, or absent in the rendered trajectory. This helps distinguish failures caused by missing evidence from those caused by incomplete grounding or incorrect observations. For aggregate reporting, we convert these labels into a soft milestone coverage rate by assigning scores of 1.0, 0.5, 0.0, and 0.0 to \texttt{covered}, \texttt{partial}, \texttt{incorrect}, and \texttt{missing}, respectively.

Second, the judge assigns zero or more closed-vocabulary failure tags that summarize common error modes, such as misunderstood conditions, missed evidence, visual misrecognition, OCR errors, counting errors, insufficient evidence-completeness verification, wrong arithmetic, or evidence--answer conflict. The vocabulary was built by inspecting baseline trajectories produced on AgentVidBench and encoding the recurring failure phenomena as a set of 21 tags. Together, milestone coverage labels and failure tags make it possible to analyze not only how well a system performs, but also why its trajectory succeeds or fails.

\subsection{Judge Validation}
We validate the trajectory judge along four lines: separation of correct and incorrect trajectories, stability across judge models and prompts, robustness to superficially plausible trajectories, and agreement with human ratings.

\paragraph{Correctness Separation.}
Because the judge never sees the answer key, answer correctness, which is simply whether the selected option matches the ground-truth option, provides an external signal against which the trajectory score can be checked. We take the 600 trajectories produced by Single-Turn, AVP, and Ours with the Gemini 2.5 Flash and Gemini 2.5 Pro backbones on all 100 questions, split them by whether the final answer was correct, and compare their trajectory scores. To check that the result does not depend on one judge, we score every trajectory with three judge models. Under every judge, correct runs receive substantially higher trajectory scores than incorrect runs. Under Claude Opus 4.7, correct runs receive a mean score of $0.922 \pm 0.128$, compared with $0.456 \pm 0.155$ for incorrect runs, a gap of $0.47$. Under GPT-5, the corresponding scores are $0.856 \pm 0.155$ and $0.501 \pm 0.186$, respectively, and under Gemini 3.1 Pro\,\cite{deepmind2026gemini31pro} they are $0.941 \pm 0.117$ and $0.571 \pm 0.164$, respectively. The area under the ROC curve, that is, the probability that a randomly chosen correct run outscores a randomly chosen incorrect run, is between $0.93$ and $0.98$ across the three judges. The trajectory score thus tracks whether the agent actually reached the right answer, even though the judge was never told which answer is right.

\paragraph{Judge and Prompt Sensitivity.}
A trajectory score is only useful if it does not depend on which judge model or which exact wording of the rubric happens to be used. We therefore test both. First, we take the 600 trajectories of the previous paragraph, each scored by Claude Opus 4.7, GPT-5, and Gemini 3.1 Pro, and compare the scores that different judges assign to the same trajectory. Between any two judges, the mean absolute difference between the scores assigned to the same trajectory is $0.10$ to $0.11$, and the two judges' per-trajectory scores correlate at $r = 0.83$ to $0.86$. This judge-to-judge difference is only about one fifth of the $0.47$ gap between correct and incorrect runs under Claude Opus 4.7 reported in the preceding paragraph, so the systems' rankings remain nearly identical regardless of the judge. Second, we test the wording of the rubric by paraphrasing it. We had Claude Opus 4.8\,\cite{anthropic2026claudeopus48} paraphrase the explanatory sentences of the judge prompt while leaving the five-axis structure, scoring criteria, examples, applicability logic, failure-tag vocabulary, and output schema unchanged. Using the same judge (Claude Opus 4.7) as in our main experiments, we then score three sets of trajectories under the original and the paraphrased rubric, namely, Single-Turn with Gemini 2.5 Flash, Single-Turn with Gemini 2.5 Pro, and Ours with Gemini 2.5 Flash, about 100 questions each. Per-question scores under the two rubrics correlate at $r = 0.96$ to $0.98$, and the mean absolute change is at most $0.036$. The trajectory score is thus stable under both the choice of judge model and the phrasing of the rubric.

\paragraph{Adversarial Trajectory Tests.}
A judge that rewarded fluent and confident writing rather than actual evidence acquisition would be easy to game. We test this with synthetic trajectories, which are scored with the same judge and the same milestones as genuine trajectories, without any indication that they are synthetic. The first test asks whether stating the right answer is sufficient. For 50 difficulty-stratified questions, we write a trajectory that confidently asserts the correct answer with generic but plausible reasoning and contains no tool calls or observations, and compare it with our agent's genuine trajectory for the same question. The synthetic trajectories receive a mean score of $0.096$, compared with $0.758$ for the genuine trajectories, and the two sets do not overlap (AUC $1.000$). The second test asks whether the judge notices when a specific piece of evidence is missing. Starting from a genuine trajectory produced by our agent, we remove everything that supports one milestone, including the tool calls that acquired it, the resulting observations, and every later mention of it, while leaving the rest of the trajectory unchanged. The score drops by $0.41$, and in $73\%$ of cases, the coverage label of exactly the removed milestone changes from \texttt{covered} to a lower label. The judge therefore responds to milestone-specific evidence rather than narrative fluency. One limitation remains by construction: a trajectory written after the fact by someone who already knows the answer, and who therefore reproduces every required piece of evidence, cannot be reliably distinguished from a genuine trajectory by trajectory-only evaluation. This case lies outside our protocol, in which agents have access to neither the answer key nor the milestones, and trajectories are produced step by step during the investigation itself, so that every observation comes from an actual tool call.

\paragraph{Human--Judge Alignment.}
Finally, correlation with answer correctness does not by itself show that the judge scores each axis in a human-aligned manner, so we compare the judge directly with human ratings. We select 20 questions stratified by difficulty and collect the trajectories generated by all five methods for each question (100 trajectories in total). Two authors then rate these trajectories on the five axes. The raters are blind to the system identity, follow the same rubric as the judge, and see only the information available to the judge, which excludes the gold answer. Human scores correlate strongly with the judge's overall trajectory scores, with Pearson $r = 0.77$. Human--judge correlation is positive on every axis: task understanding $r = 0.39$, evidence coverage $r = 0.71$, evidence grounding $r = 0.75$, evidence completeness $r = 0.60$ on its applicable subset, and reasoning faithfulness $r = 0.58$. The relatively lower correlation for task understanding should be interpreted in light of a ceiling effect: scores on this axis are at the maximum for most trajectories under both human and judge ratings, leaving little variance and making the correlation sensitive to the few non-maximal cases. The overall correlation is higher than the per-axis correlations because averaging across axes reduces rating noise relative to comparisons on individual axes. At the method level, human and judge scores yield closely aligned rankings and identify the same top-ranked method.

\subsection{Model Settings and Prompt Templates}
\label{app:prompt-settings}
To support reproducibility, we provide the model configurations and prompt templates used for milestone extraction and trajectory judging. The final pipeline uses GPT-5 for extracting milestones and Claude Opus 4.7 for judging trajectories. All trajectory judge calls use milestone-only references.

\begin{table}[h!]
    \centering
    \caption{Model settings used in milestone extraction and trajectory judging.}
    \label{tab:llm-settings}
    \small
    \setlength{\tabcolsep}{6pt}
    \renewcommand{\arraystretch}{1.12}
    \begin{tabular}{p{0.17\linewidth}p{0.15\linewidth}p{0.57\linewidth}}
    \toprule
    Component & Model & Settings \\
    \midrule
    Milestone extraction &
    GPT-5 &
    Reasoning effort: medium; maximum completion tokens: 16,384. \\

    Trajectory judge &
    Claude Opus 4.7 &
    1M context enabled; maximum output tokens: 4,096. \\
    \bottomrule
    \end{tabular}
\end{table}

\paragraph{Milestone Extraction Prompt.}
The milestone extraction prompt converts a human-written ground-truth reasoning note into a structured, tool-agnostic list of evidence milestones. 
\begin{tcolorbox}[
        title=Milestone extraction system prompt,
        colback=gray!4,
        colframe=gray!40,
        coltitle=black,
        colbacktitle=gray!12,
        fonttitle=\bfseries,
        boxrule=0.4pt,
        arc=2pt,
        breakable
    ]
    \begin{quote}
    \small
    \textbf{System prompt.}\\
    You are abstracting a video QA question's human-written solution path into tool-agnostic evidence milestones.

    Different agents use different tool interfaces, such as low-level video controls, high-level video-analysis calls, or zero-shot video input. Comparing raw tool sequences would be unfair. Extract instead the underlying evidence the agent must acquire to reach the answer: what to find, not how to find it.

    \textbf{Inputs.}\\
    You will receive the question text, answer options, gold answer, gold answer text, and the human-written answer explanation.

    \textbf{Output.}\\
    Return a JSON object with a \texttt{milestones} array. Each milestone contains:
    \begin{itemize}[leftmargin=1em]
        \item \texttt{id}: \texttt{M1}, \texttt{M2}, ... starting from 1.
        \item \texttt{type}: one of \texttt{context\_identification}, \texttt{temporal\_localization}, \texttt{ocr\_text\_reading}, \texttt{audio\_visual\_evidence}, \texttt{visual\_recognition}, \texttt{counting}, \texttt{comparison}, \texttt{exclusion\_check}, \texttt{arithmetic}, or \texttt{verification}.
        \item \texttt{description}: a one-line English description of the evidence to acquire.
    \end{itemize}

    \textbf{Core principles.}\\
    Milestones must be tool-agnostic. Do not mention tool names such as OCR tools or video-analysis calls. Do not include meta-references such as ``in the ground-truth solution.'' Do not bake specific timestamps, narrow time ranges, timestamp lists, or solution-path-specific arithmetic decompositions into the description. Use descriptions such as ``Locate the moment when the target object is visible'' rather than ``Inspect the video at 02:13.''

    Do not include a final option-mapping milestone. Stop at the value derivation or evidence requirement; mapping to an option letter is not treated as evidence.

    For counting questions, do not create one milestone per occurrence. Group the target as a single counting milestone and use a separate exclusion-check milestone when the question requires excluding distractors or duplicates.

    For OCR or identification questions, include the value to be read when that value is necessary for the answer. For arithmetic or comparison questions, include the required operation at the evidence level without prescribing a particular human decomposition unless the decomposition itself is required by the question.

    Return JSON only, with no markdown fence and no preamble.
    \end{quote}
\end{tcolorbox}

\paragraph{Trajectory Judge Prompt.}
The trajectory judge receives a milestone-only reference in the default setting. The original human answer explanation is not included in the judge input, so that the judge evaluates whether the trajectory recovers the underlying evidence rather than whether it follows the same human solution path.

\begin{tcolorbox}[
        title=Trajectory judge system prompt,
        colback=gray!4,
        colframe=gray!40,
        coltitle=black,
        colbacktitle=gray!12,
        fonttitle=\bfseries,
        boxrule=0.4pt,
        arc=2pt,
        breakable
    ]
    \begin{quote}
    \small
    \textbf{System prompt.}\\
    You are a strict evaluator of video-investigation trajectories on AgentVidBench.

    This call evaluates the five process-quality axes, P1--P5, together with milestone coverage labels and failure tags. Score each applicable axis on a 0--2 scale, where 0 is worst and 2 is best. P4 may be \texttt{null} when the question does not require a full-scope evidence sweep.

    Follow the tool-agnostic principle: do not penalize differences in tool names, tool interfaces, or search paths. The agent may have arrived at the same evidence via a different path.

    Use \texttt{GT\_MILESTONES} as the sole reference for evidence. Do not use anything else as a reference.

    \textbf{Process axes.}
    \begin{itemize}[leftmargin=1em]
        \item \textbf{P1 Task Understanding}: did the agent correctly understand WHAT the question asks --- target entity, constraints, exclusions, temporal scope, and option-letter range?
        \begin{itemize}[leftmargin=1em]
            \item \textbf{2}: agent's framing matches the question exactly; all constraints (exclusions, qualifiers, ``first/last'', option range A--Z) acknowledged either explicitly or via consistent behavior.
            \item \textbf{1}: core task understood but a constraint or qualifier missed (e.g., misses an exclusion, treats ``first'' as ``any'').
            \item \textbf{0}: fundamental misunderstanding --- wrong target entity, wrong question type, or ignores a critical constraint that determines the answer.
        \end{itemize}

        \item \textbf{P2 Ground-truth Evidence Coverage}: how many of the \texttt{GT\_MILESTONES} did the agent's trajectory ACQUIRE? Use the milestone\_coverage classifications as the basis. All milestones treated uniformly (no required/convenient distinction).
        \begin{itemize}[leftmargin=1em]
            \item \textbf{2}: all milestones covered (or covered via tool-agnostic alternate paths).
            \item \textbf{1}: some milestones missing or partially covered, but enough acquired that the answer was reachable.
            \item \textbf{0}: most milestones missing or incorrect; the agent's answer is essentially unsupported by acquired evidence.
        \end{itemize}

        \item \textbf{P3 Evidence Grounding}: are the values the agent CLAIMS (OCR readings, counts, observed objects, audio quotes, arithmetic intermediate values) consistent with what the milestones say?
        \begin{itemize}[leftmargin=1em]
            \item \textbf{2}: every claimed value matches the milestone evidence; no fabrication, no misreading.
            \item \textbf{1}: most values correct; at least one mismatch (wrong OCR digit, off-by-N count, mis-identified object) but the mismatch is non-fatal.
            \item \textbf{0}: multiple claimed values wrong, OR the value that determines the final answer is wrong/fabricated.
        \end{itemize}

        \item \textbf{P4 Exhaustive Evidence Sweep} (may be \texttt{null}): did the agent verify that the relevant temporal scope was adequately inspected before committing? The user prompt declares \texttt{SWEEP\_REQUIRED} per qid (\texttt{yes}/\texttt{no}/\texttt{unknown}). If \texttt{no} (anchored question, single-segment determination), set this axis to \texttt{null} and skip. If \texttt{yes}:
        \begin{itemize}[leftmargin=1em]
            \item \textbf{2}: agent explicitly extends inspection past the first match AND confirms no additional / contradicting / qualifying evidence in remaining temporal scope, citing concrete inspected scope.
            \item \textbf{1}: extends past first match but partial --- scans later portion without explicit FULL-scope confirmation, OR mentions later events as a possibility without inspecting them.
            \item \textbf{0}: commits at first match without any check. Verbal confidence (``clearly'', ``I'm sure'') does NOT count as a sweep.
        \end{itemize}

        \item \textbf{P5 Reasoning Faithfulness}: does each reasoning step follow from the evidence the agent claimed? Is the arithmetic correct? Are there unsupported leaps from evidence to the final answer?
        \begin{itemize}[leftmargin=1em]
            \item \textbf{2}: each step traceable to acquired evidence; arithmetic correct; no unsupported leaps; final answer follows from the chain.
            \item \textbf{1}: minor leap or arithmetic slip but core flow defensible; final answer still consistent with the trajectory's evidence.
            \item \textbf{0}: reasoning contradicts own evidence (claims X then concludes $\neg$X), OR commits a fatal arithmetic / logic error, OR final answer is disconnected from the evidence acquired.
        \end{itemize}
    \end{itemize}

    \textbf{Milestone coverage.}\\
    For each reference milestone (fixed IDs M1..Mn given in the user prompt), assign one of four labels: \texttt{covered}, \texttt{partial}, \texttt{incorrect}, or \texttt{missing}. The output must list the same milestones in the same order. These labels feed an auxiliary mc\_rate aggregate (\texttt{covered}=1, \texttt{partial}=0.5, \texttt{incorrect}=0, \texttt{missing}=0) reported alongside the P-axis scores.

    \textbf{Failure tags.}\\
    Select zero to eight tags from the closed vocabulary:
    \begin{center}
    \footnotesize
    \setlength{\tabcolsep}{8pt}
    \renewcommand{\arraystretch}{1.05}
    \begin{tabular}{@{}ll@{}}
    \texttt{misunderstood\_condition} & \texttt{wrong\_target\_entity} \\
    \texttt{visual\_misrecognition} & \texttt{ocr\_error} \\
    \texttt{audio\_transcript\_error} & \texttt{unsupported\_observation} \\
    \texttt{counting\_error} & \texttt{duplicate\_counting} \\
    \texttt{excluded\_item\_counted} & \texttt{wrong\_arithmetic} \\
    \texttt{wrong\_option\_mapping} & \texttt{wrong\_temporal\_segment} \\
    \texttt{missed\_required\_segment} & \texttt{insufficient\_full\_video\_scan} \\
    \texttt{spatial\_relation\_error} & \texttt{premature\_fixation} \\
    \texttt{failure\_to\_expand\_search} & \texttt{failure\_to\_adjust\_granularity} \\
    \texttt{evidence\_answer\_conflict} & \texttt{premature\_answer} \\
    \texttt{overconfident\_uncertainty} & \\
    \end{tabular}
    \end{center}

    The predicted answer letter is computed externally; this judge does NOT score answer correctness.

    \textbf{Output.} Return strict JSON only.
    \end{quote}
\end{tcolorbox}

\paragraph{Judge User Template.}
For each question--prediction pair, we instantiate the following user prompt.

\begin{tcblisting}{
    title=Trajectory judge user template,
    colback=gray!2,
    colframe=gray!25,
    colbacktitle=gray!8,
    coltitle=black,
    fonttitle=\bfseries\small,
    boxrule=0.35pt,
    arc=1.5pt,
    left=5pt,
    right=5pt,
    top=5pt,
    bottom=5pt,
    before skip=6pt,
    after skip=8pt,
    width=0.90\linewidth,
    center,
    listing only,
    breakable,
    listing options={
        basicstyle=\ttfamily\footnotesize,
        breaklines=true,
        columns=fullflexible,
        keepspaces=true
    }
}
QUESTION_ID: {qid}
FRAMEWORK: {framework}
SWEEP_REQUIRED: {sweep_required}

QUESTION:
{question_text}

OPTIONS:
{options}

GT_MILESTONES (curated, fixed IDs M1..Mn --- classify coverage status for EACH, in order):
{milestones_block}

PREDICTED_TRAJECTORY:
{prediction}

Evaluate this prediction against the rubric. Output the JSON object only.
Reminder: milestone_coverage MUST contain exactly {n_milestones}
entries with IDs {milestone_ids} in this order.
\end{tcblisting}

\paragraph{Judge Output Schema.}
The judge returns a strict JSON object with process scores, short rationales, auxiliary milestone coverage labels, failure tags, and a concise summary.

\begin{tcblisting}{
    title=Trajectory judge output schema,
    colback=gray!2,
    colframe=gray!25,
    colbacktitle=gray!8,
    coltitle=black,
    fonttitle=\bfseries\small,
    boxrule=0.35pt,
    arc=1.5pt,
    left=5pt,
    right=5pt,
    top=5pt,
    bottom=5pt,
    before skip=6pt,
    after skip=8pt,
    width=0.90\linewidth,
    center,
    listing only,
    breakable,
    listing options={
        basicstyle=\ttfamily\footnotesize,
        breaklines=true,
        columns=fullflexible,
        keepspaces=true
    }
}
{
  "process_scores": {
    "task_understanding": 0,
    "gt_evidence_coverage": 0,
    "evidence_grounding": 0,
    "evidence_completeness": null,
    "reasoning_faithfulness": 0
  },
  "axis_rationales": {
    "P1": "...",
    "P2": "...",
    "P3": "...",
    "P4": "...",
    "P5": "..."
  },
  "milestone_coverage": [
    {
      "milestone_id": "M1",
      "status": "covered",
      "evidence_in_prediction": "..."
    }
  ],
  "failure_tags": ["..."],
  "summary": "..."
}
\end{tcblisting}

\section{Appendix -- Our Proposed Agentic Framework}
\label{app:ours}

We provide the technical details of Video-ReAct (Ours). Our approach employs a lightweight, ReAct-style agentic framework. At each reasoning step, the main planner evaluates the current context to identify missing evidence and dynamically invokes one of two external tools. The first, \texttt{analyze\_video(start\_time, end\_time, fps $\in [1,20]$, focus, resolution)}, is a fine-grained video inspection tool consistently powered by Gemini 2.5 Pro across all evaluated backbones to maintain a controlled perceptual baseline. The second, \texttt{get\_transcript(start, end)}, is a text-retrieval tool that accesses pre-extracted audio transcripts generated by Whisper-large-v3\,\cite{DBLP:conf/icml/RadfordKXBMS23}. The iterative execution loop is bounded to a maximum of 20 reasoning iterations and 15 tool calls per query, and the complete system prompt is provided in our source code. A qualitative example of the complete trajectory is shown in Figure~\ref{fig:ours_qualitative}.

\begin{figure}[p]
  \centering
  \vspace{-0.4cm}
  \includegraphics[width=0.97\textwidth]{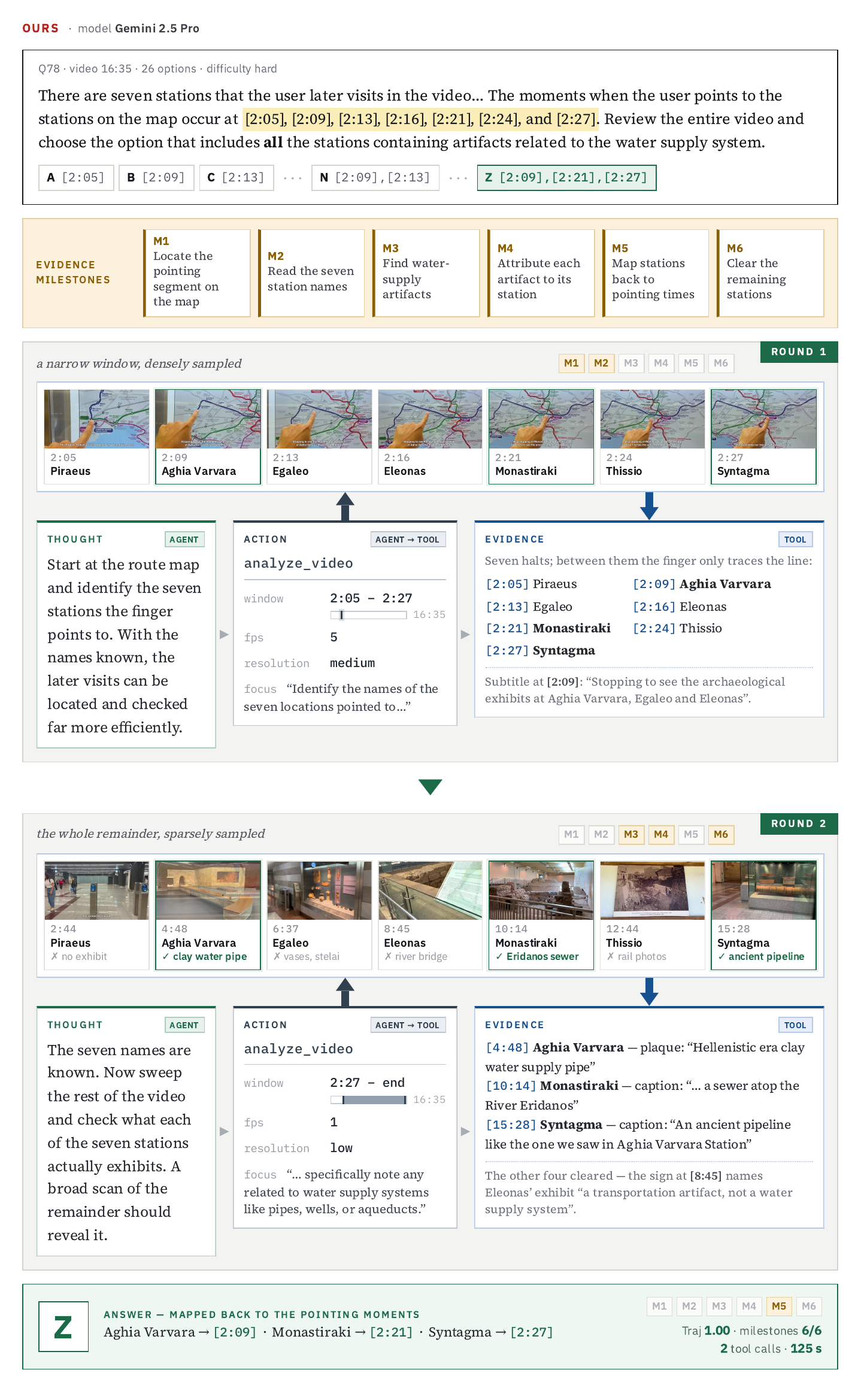}
  \vspace{-0.4cm}
  \caption{\textbf{Qualitative example of Video-ReAct (Ours).} Round~1 spends a dense budget on a narrow window (2:05--2:27, 5\,FPS, medium) to read the seven station names off the route map; Round~2 inverts the setting (2:27--end, 1\,FPS, low) to survey every station visit at once, grounding water-supply artifacts at three stations and clearing the other four.}
  \label{fig:ours_qualitative}
\end{figure}

\clearpage
\newpage

\section{Appendix -- Auxiliary Trajectory Diagnostics}
\label{app:trajectory-diagnostics}

We provide additional trajectory-level diagnostics beyond the
aggregate trajectory score reported in the main text. These diagnostics are
derived from the same milestone-based trajectory judge described in
Appendix~\ref{app:evaluation-pipeline}, but they are used only for analysis and
are not included in the computation of \emph{Traj.}

\subsection{Detailed Process-quality Profiles}
\label{app:detailed-trajectory-profiles}

Figures~\ref{fig:app-radar-proprietary}, \ref{fig:app-radar-qwen}, and
\ref{fig:app-radar-gemma} provide the full trajectory process-quality profiles
for proprietary MLLMs, Qwen-family open-source MLLMs, and Gemma-family
open-source MLLMs, respectively. These figures complement the aggregate
trajectory scores reported in the main text by showing how each method differs
across the five process-quality axes: Task Understanding (TU), Evidence
Coverage (ECov), Evidence Grounding (EG), Evidence Completeness (ECom), and
Reasoning Faithfulness (RF).

\begin{figure}[h]
    \centering
    \includegraphics[
        width=\textwidth,
        height=0.80\textheight,
        keepaspectratio
    ]{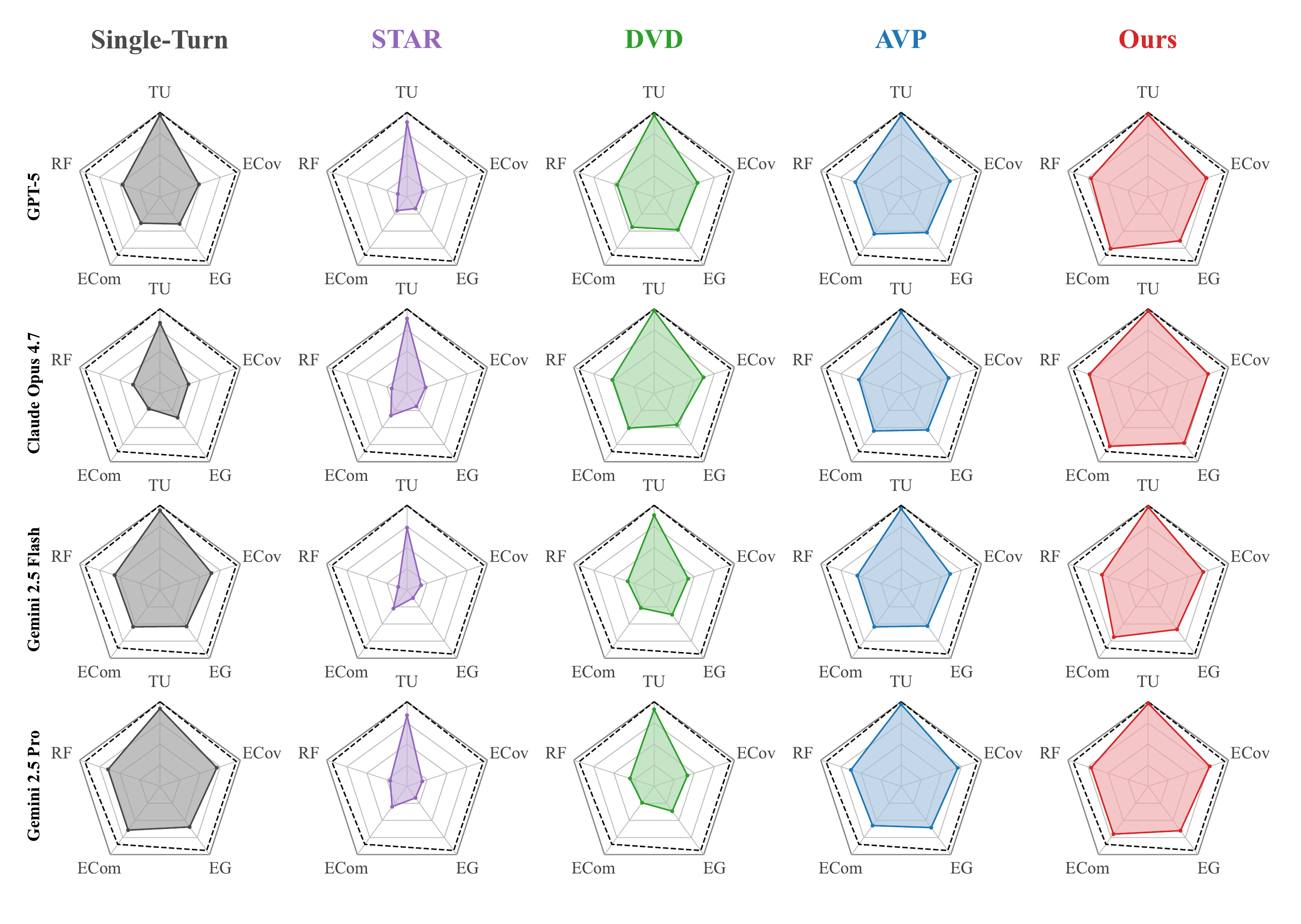}
    \caption{
    Trajectory score profiles of proprietary MLLMs under Single-Turn and agentic frameworks on AgentVidBench.
    Each radar plot shows the average scores across five trajectory evaluation dimensions:
    Task Understanding (TU), Evidence Coverage (ECov), Evidence Grounding (EG),
    Evidence Completeness (ECom), and Reasoning Faithfulness (RF).
    Columns correspond to Single-Turn, STAR, DVD, AVP, and Ours.
    The dashed outline indicates the human reference score, where larger filled regions
    represent stronger trajectory quality.
    }
    \label{fig:app-radar-proprietary}
\end{figure}

\begin{figure}[t]
    \centering
    \includegraphics[
        width=\textwidth
    ]{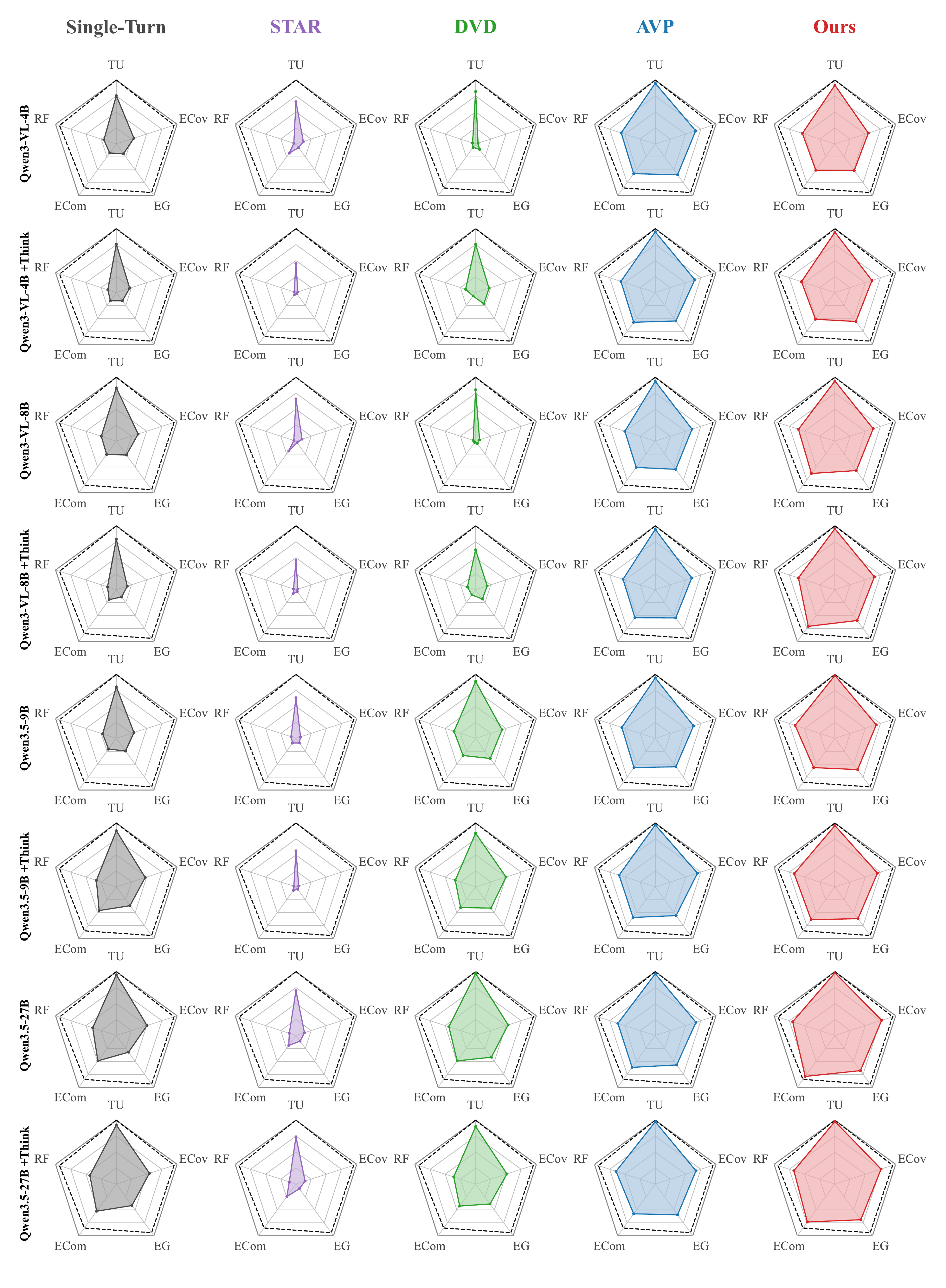}
    \caption{
    Trajectory score profiles of Qwen-family open-source MLLMs under Single-Turn and agentic frameworks on AgentVidBench.
    Each row corresponds to a Qwen backbone or its thinking variant, and each column corresponds
    to Single-Turn, STAR, DVD, AVP, or Ours.
    Each radar plot shows the average scores across five trajectory evaluation dimensions:
    Task Understanding (TU), Evidence Coverage (ECov), Evidence Grounding (EG),
    Evidence Completeness (ECom), and Reasoning Faithfulness (RF).
    The dashed outline indicates the human reference score, where larger filled regions
    represent stronger trajectory quality.
    }
    \label{fig:app-radar-qwen}
\end{figure}

\begin{figure}[t]
    \centering
    \includegraphics[
        width=\textwidth
    ]{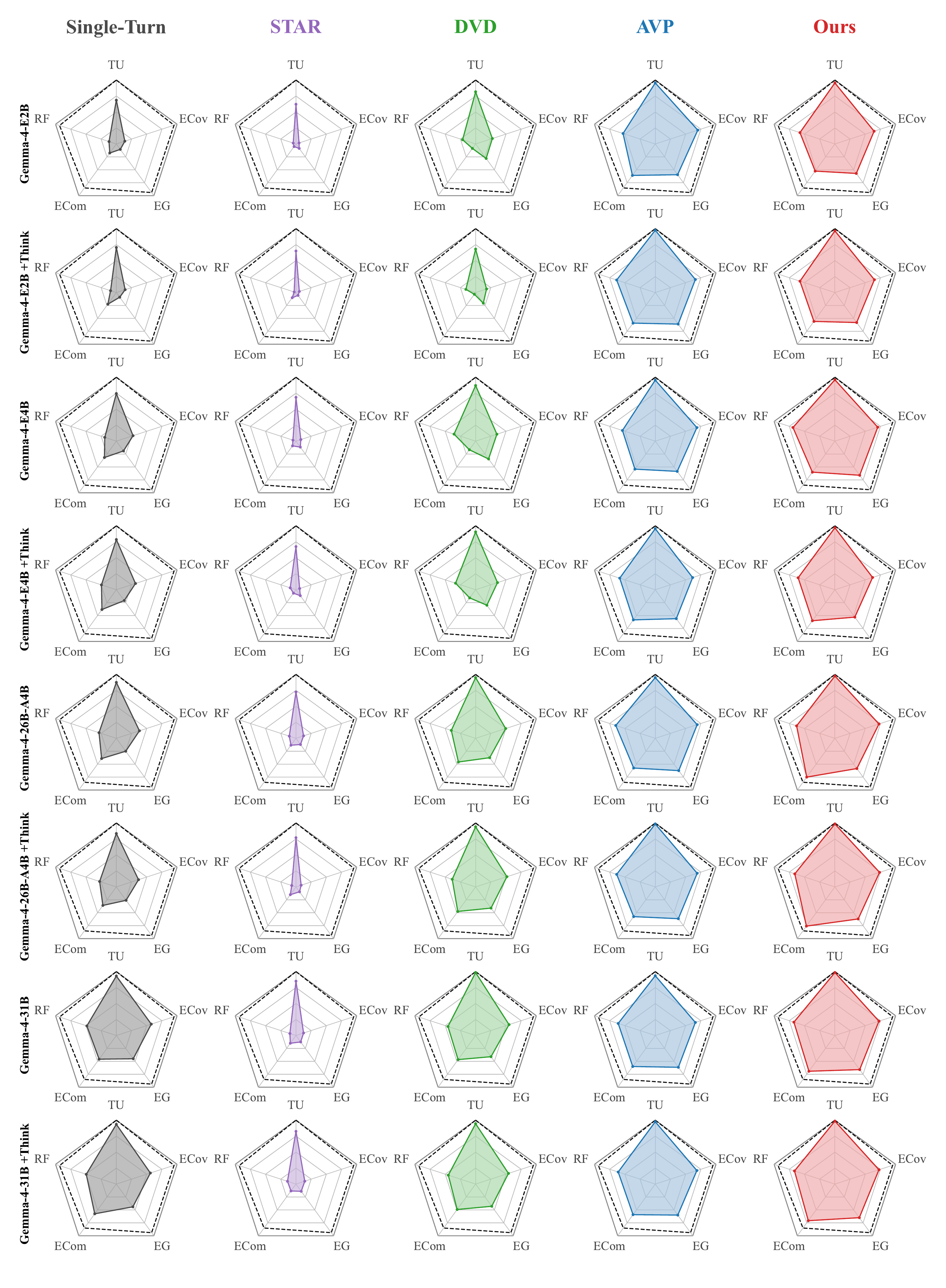}
    \caption{
    Trajectory score profiles of Gemma-family open-source MLLMs under Single-Turn and agentic frameworks on AgentVidBench.
    Each row corresponds to a Gemma backbone or its thinking variant, and each column corresponds
    to Single-Turn, STAR, DVD, AVP, or Ours.
    Each radar plot shows the average scores across five trajectory evaluation dimensions:
    Task Understanding (TU), Evidence Coverage (ECov), Evidence Grounding (EG),
    Evidence Completeness (ECom), and Reasoning Faithfulness (RF).
    The dashed outline indicates the human reference score, where larger filled regions
    represent stronger trajectory quality.
    }
    \label{fig:app-radar-gemma}
\end{figure}

\clearpage
\newpage

\subsection{Milestone Coverage Rate}
\label{app:milestone-coverage-rate}

In addition to the five process-quality axes, the trajectory judge assigns a coverage label to each reference milestone: \texttt{covered}, \texttt{partial}, \texttt{incorrect}, or \texttt{missing}. We report a soft milestone coverage rate that assigns full credit to \texttt{covered} milestones, half credit to \texttt{partial} milestones, and zero credit to \texttt{incorrect} or \texttt{missing} milestones:
\[
    \mathrm{MC\text{-}Rate}
    =
    \frac{
        n_{\mathrm{covered}} + 0.5\,n_{\mathrm{partial}}
    }{
        N_{\mathrm{milestones}}
    }.
\]
Here, \(N_{\mathrm{milestones}}\) denotes the total number of reference milestones. Unlike the main trajectory score, this metric does not average over the five process-quality axes. It provides a direct diagnostic view of how much of the reference evidence each system recovers at the milestone level. Table~\ref{tab:mc_rate} reports the milestone coverage rate for every backbone and method in Table~\ref{tab:main_results}.

\begin{table*}[h]
\centering
\caption{Soft milestone coverage rate on the AgentVidBench dataset, reported for each combination of backbone and method. The metric assigns full credit to \texttt{covered} milestones, half credit to \texttt{partial} milestones, and zero credit to \texttt{incorrect} or \texttt{missing} milestones. Blue arrows indicate absolute changes relative to the corresponding Single-Turn setting.}
\small
\setlength{\tabcolsep}{6pt}
\renewcommand{\arraystretch}{1.12}
\begin{tabular}{l c c c c | c}
\toprule
\textbf{Models} & \textbf{Single-Turn} & \textbf{STAR}~\cite{fan2025toolaugmented} & \textbf{DVD}~\cite{DBLP:conf/nips/ZhangJGLLLL25} & \textbf{AVP}~\cite{wang2025activevideoperceptioniterative} & \textbf{Ours} \\
\midrule
\rowcolor{gray!18}
\multicolumn{6}{@{}l}{\textbf{Proprietary MLLMs}} \\
GPT-5              & 0.58 & 0.29 & 0.58 & 0.66 & \textbf{0.75} \textcolor{blue}{\scriptsize{$\uparrow$0.17}} \\
Claude Opus 4.7    & 0.42 & 0.31 & 0.63 & 0.64 & \textbf{0.78} \textcolor{blue}{\scriptsize{$\uparrow$0.36}} \\
Gemini 2.5 Flash   & 0.68 & 0.25 & 0.44 & 0.66 & \textbf{0.71} \textcolor{blue}{\scriptsize{$\uparrow$0.03}} \\
Gemini 2.5 Pro     & 0.73 & 0.28 & 0.45 & 0.74 & \textbf{0.79} \textcolor{blue}{\scriptsize{$\uparrow$0.06}} \\
\midrule
\rowcolor{gray!18}
\multicolumn{6}{@{}l}{\textbf{Open-source MLLMs}} \\
\rowcolor{gray!7}
\multicolumn{6}{@{}l}{\quad\textit{Qwen family}} \\
Qwen3-VL-4B                  & 0.34 & 0.16 & 0.07 & \textbf{0.70} & 0.60 \textcolor{blue}{\scriptsize{$\uparrow$0.26}} \\
\quad\textit{+Think}         & 0.28 & 0.04 & 0.24 & \textbf{0.70} & 0.63 \textcolor{blue}{\scriptsize{$\uparrow$0.35}} \\
Qwen3-VL-8B                  & 0.42 & 0.16 & 0.10 & 0.64 & \textbf{0.67} \textcolor{blue}{\scriptsize{$\uparrow$0.25}} \\
\quad\textit{+Think}         & 0.27 & 0.05 & 0.24 & 0.67 & \textbf{0.67} \textcolor{blue}{\scriptsize{$\uparrow$0.40}} \\
Qwen3.5-9B                   & 0.36 & 0.13 & 0.50 & 0.67 & \textbf{0.70} \textcolor{blue}{\scriptsize{$\uparrow$0.35}} \\
\quad\textit{+Think}         & 0.56 & 0.12 & 0.51 & 0.70 & \textbf{0.73} \textcolor{blue}{\scriptsize{$\uparrow$0.17}} \\
Qwen3.5-27B                  & 0.56 & 0.21 & 0.59 & 0.70 & \textbf{0.80} \textcolor{blue}{\scriptsize{$\uparrow$0.23}} \\
\quad\textit{+Think}         & 0.62 & 0.24 & 0.54 & 0.71 & \textbf{0.78} \textcolor{blue}{\scriptsize{$\uparrow$0.16}} \\
\midrule
\addlinespace[2pt]
\rowcolor{gray!7}
\multicolumn{6}{@{}l}{\quad\textit{Gemma family}} \\
Gemma-4-E2B                  & 0.21 & 0.11 & 0.32 & \textbf{0.71} & 0.68 \textcolor{blue}{\scriptsize{$\uparrow$0.47}} \\
\quad\textit{+Think}         & 0.21 & 0.10 & 0.22 & \textbf{0.71} & 0.68 \textcolor{blue}{\scriptsize{$\uparrow$0.47}} \\
Gemma-4-E4B                  & 0.34 & 0.15 & 0.42 & 0.70 & \textbf{0.74} \textcolor{blue}{\scriptsize{$\uparrow$0.40}} \\
\quad\textit{+Think}         & 0.38 & 0.14 & 0.40 & \textbf{0.68} & 0.66 \textcolor{blue}{\scriptsize{$\uparrow$0.28}} \\
Gemma-4-26B-A4B              & 0.44 & 0.21 & 0.53 & 0.73 & \textbf{0.76} \textcolor{blue}{\scriptsize{$\uparrow$0.32}} \\
\quad\textit{+Think}         & 0.44 & 0.16 & 0.57 & 0.72 & \textbf{0.77} \textcolor{blue}{\scriptsize{$\uparrow$0.33}} \\
Gemma-4-31B                  & 0.61 & 0.20 & 0.57 & 0.71 & \textbf{0.76} \textcolor{blue}{\scriptsize{$\uparrow$0.15}} \\
\quad\textit{+Think}         & 0.62 & 0.23 & 0.60 & 0.70 & \textbf{0.75} \textcolor{blue}{\scriptsize{$\uparrow$0.13}} \\
\bottomrule
\end{tabular}
\label{tab:mc_rate}
\end{table*}

\subsection{Failure-Mode Analysis}
\label{app:failure-modes}
Alongside the five process-quality axes, the trajectory judge assigns each trajectory zero or more tags from the closed vocabulary of 21 failure tags introduced in Appendix~\ref{app:auxiliary-diagnostics} and listed in full in the judge prompt of Appendix~\ref{app:prompt-settings}. Table~\ref{tab:failure-tags} reports the frequency of every tag per 100 questions for each evaluated setting, averaged over the 20 backbone configurations, and Table~\ref{tab:coverage-by-type} breaks the soft milestone coverage rate of Table~\ref{tab:mc_rate} down by milestone type.

Two failure mechanisms recur across settings. The first is accepting evidence without verification, either by asserting it outright, which is the dominant failure of single-pass models, or by handing it unchecked between agent components. The second is ending the sweep before the required temporal scope has been covered. The most frequent tag of each setting makes the contrast concrete. For Single-Turn models, it is \texttt{unsupported\_observation} at 47.1 occurrences per 100 questions: with no tools, missing evidence is asserted rather than sought. For STAR, the same tag reaches 70.2, the highest of any setting, together with the highest rates of \texttt{premature\_answer}, 42.5 versus 8.4 for Single-Turn, and \texttt{evidence\_answer\_conflict} at 65.0, which is consistent with a fixed schedule that ends the search on plan rather than on evidence. For DVD, it is \texttt{insufficient\_full\_video\_scan} at 43.1, consistent with retrieval that surfaces plausible clips without guaranteeing coverage of the required intervals. For AVP and for our agent, the dominant tag shifts to \texttt{counting\_error} at 31.9 and 29.3 respectively: once evidence seeking replaces assertion, frame-sampled perception becomes the bottleneck, and the perception tags \texttt{counting\_error}, \texttt{visual\_misrecognition}, and \texttt{ocr\_error} remain the largest residual failures even for the best settings. Our agent shows among the lowest rates on most tags and the highest coverage on every milestone type in Table~\ref{tab:coverage-by-type}, yet its coverage stays lowest on \texttt{counting} at 0.45, \texttt{verification} at 0.56, and \texttt{exclusion\_check} at 0.64, indicating concrete headroom on the benchmark.

\begin{table}[h]
\centering
\caption{Failure-tag frequency, in occurrences per 100 questions, by evaluated setting, averaged over the 20 backbone configurations. Tags are ordered by their overall frequency. The lowest frequency in each row is shown in bold.}
\label{tab:failure-tags}
\small
\begin{tabular}{lccccc}
\toprule
Failure tag & Single-Turn & STAR~\cite{fan2025toolaugmented} & DVD~\cite{DBLP:conf/nips/ZhangJGLLLL25} & AVP~\cite{wang2025activevideoperceptioniterative} & Ours \\
\midrule
\texttt{unsupported\_observation} & 47.1 & 70.2 & 34.5 & 15.9 & \textbf{7.7} \\
\texttt{counting\_error} & 37.8 & 38.6 & 37.0 & 31.9 & \textbf{29.3} \\
\texttt{evidence\_answer\_conflict} & 27.3 & 65.0 & 32.7 & \textbf{17.6} & \textbf{17.6} \\
\texttt{insufficient\_full\_video\_scan} & 27.5 & 45.0 & 43.1 & 22.1 & \textbf{14.7} \\
\texttt{visual\_misrecognition} & 37.3 & 34.4 & 23.6 & 23.9 & \textbf{23.0} \\
\texttt{premature\_fixation} & 25.2 & 23.4 & 30.9 & \textbf{8.4} & 15.2 \\
\texttt{ocr\_error} & 20.4 & 27.0 & 16.7 & 12.6 & \textbf{11.3} \\
\texttt{premature\_answer} & 8.4 & 42.5 & 25.1 & \textbf{1.9} & 5.0 \\
\texttt{missed\_required\_segment} & 9.4 & 19.0 & 15.2 & \textbf{5.3} & 5.7 \\
\texttt{wrong\_temporal\_segment} & 14.2 & 14.6 & 10.2 & \textbf{4.3} & 5.6 \\
\texttt{wrong\_option\_mapping} & 6.5 & 15.2 & 8.2 & 9.1 & \textbf{5.3} \\
\texttt{overconfident\_uncertainty} & 13.7 & 8.8 & 6.8 & \textbf{3.0} & 6.8 \\
\texttt{failure\_to\_expand\_search} & 2.2 & 12.7 & 8.2 & \textbf{1.2} & 2.5 \\
\texttt{failure\_to\_adjust\_granularity} & 4.6 & 5.5 & 4.5 & 6.6 & \textbf{4.0} \\
\texttt{duplicate\_counting} & 5.2 & 4.5 & \textbf{1.9} & 5.0 & 7.7 \\
\texttt{wrong\_target\_entity} & 5.1 & 5.0 & 2.4 & 3.1 & \textbf{1.9} \\
\texttt{wrong\_arithmetic} & 3.0 & 6.7 & 3.3 & 2.4 & \textbf{2.1} \\
\texttt{misunderstood\_condition} & 4.5 & 2.9 & \textbf{1.7} & 4.2 & 3.2 \\
\texttt{excluded\_item\_counted} & 1.8 & 1.9 & \textbf{1.4} & 2.4 & 3.0 \\
\texttt{audio\_transcript\_error} & 3.4 & 2.6 & 1.1 & \textbf{0.4} & 0.7 \\
\texttt{spatial\_relation\_error} & 0.6 & 0.4 & \textbf{0.1} & 0.7 & 0.3 \\
\bottomrule
\end{tabular}
\end{table}

\begin{table}[h]
\centering
\caption{Soft milestone coverage rate by milestone type and evaluated setting, averaged over the 20 backbone configurations. The highest coverage in each row is shown in bold.}
\label{tab:coverage-by-type}
\small
\begin{tabular}{lccccc}
\toprule
Milestone type & Single-Turn & STAR~\cite{fan2025toolaugmented} & DVD~\cite{DBLP:conf/nips/ZhangJGLLLL25} & AVP~\cite{wang2025activevideoperceptioniterative} & Ours \\
\midrule
\texttt{temporal\_localization} & 0.56 & 0.27 & 0.53 & 0.84 & \textbf{0.86} \\
\texttt{visual\_recognition} & 0.45 & 0.18 & 0.42 & 0.67 & \textbf{0.70} \\
\texttt{exclusion\_check} & 0.39 & 0.11 & 0.29 & 0.56 & \textbf{0.64} \\
\texttt{counting} & 0.26 & 0.06 & 0.20 & 0.43 & \textbf{0.45} \\
\texttt{context\_identification} & 0.74 & 0.40 & 0.70 & 0.90 & \textbf{0.91} \\
\texttt{ocr\_text\_reading} & 0.39 & 0.05 & 0.43 & 0.67 & \textbf{0.69} \\
\texttt{verification} & 0.26 & 0.04 & 0.24 & 0.53 & \textbf{0.56} \\
\texttt{audio\_visual\_evidence} & 0.42 & 0.17 & 0.51 & 0.79 & \textbf{0.81} \\
\texttt{comparison} & 0.36 & 0.12 & 0.36 & 0.70 & \textbf{0.74} \\
\texttt{arithmetic} & 0.39 & 0.05 & 0.38 & \textbf{0.81} & 0.73 \\
\bottomrule
\end{tabular}
\end{table}

\clearpage
\newpage

\section{Appendix -- Dataset Size and Statistical Reliability}
\label{app:stat-reliability}

The scale of AgentVidBench is primarily constrained both by our strict adherence to copyright regulations, which makes finding legally unencumbered videos suitable for complex QA challenging, and by the labor-intensive nature of constructing detailed reasoning trajectories for each sample rather than simply providing final answers. While we recognize that other video QA benchmarks have achieved large scales through massive curation and labeling budgets, constructing such datasets is often prohibitive for institutions with limited resources. To mitigate this limitation, we dedicated our efforts to meticulous curation, ensuring that our 100 video-QA samples maintain representativeness by deliberately diversifying them across video genres, video durations, required skills, and question difficulty, as illustrated in Figure~\ref{fig:dataset-statistics}.

To compensate for the smaller dataset size, we include confidence intervals and statistical significance tests to rigorously validate the reported differences between models and agentic frameworks within our dataset. We evaluate the statistical reliability of the
comparisons reported in Table~\ref{tab:main_results} through bootstrap resampling
and paired significance tests. Specifically, Section~\ref{app:framework-comparisons}
examines the accuracy differences between each agentic framework and the Single-Turn
baseline, while Section~\ref{app:ranking-stability} examines the stability of the
model ranking induced by the benchmark.

\subsection{Framework Comparisons}
\label{app:framework-comparisons}

\paragraph{Paired Confidence Intervals.}
To quantify the uncertainty in the accuracy differences between each
agentic framework and the Single-Turn baseline, we report paired bootstrap confidence
intervals in Table~\ref{tab:paired_ci}. Specifically, we draw 10{,}000 bootstrap
samples of the 100 questions with replacement and recompute the accuracy difference on
each, scoring both systems on the same drawn questions. We take the middle 95\% of
these differences as the confidence interval and report the largest model of each
family as a representative. The results show that eight of the twelve intervals lie
entirely on one side of zero. These gaps therefore keep the same direction under
essentially every alternative draw of questions and are statistically significant.

\begin{table}[h]
\centering
\caption{Paired bootstrap confidence intervals for the accuracy
difference between each agentic framework and Single-Turn ($10{,}000$ resamples).
Point estimates whose interval does not contain zero are shown in bold.}
\small
\setlength{\tabcolsep}{4pt}
\renewcommand{\arraystretch}{1.12}
\begin{tabular}{@{}l cccc@{}}
\toprule
\textbf{Models}
& \textbf{STAR}~\cite{fan2025toolaugmented}
& \textbf{DVD}~\cite{DBLP:conf/nips/ZhangJGLLLL25}
& \textbf{AVP}~\cite{wang2025activevideoperceptioniterative}
& \textbf{Ours} \\
\midrule
Qwen3.5-27B    & $\mathbf{-0.16}$ {\scriptsize $[-0.25, -0.07]$} & $\mathbf{+0.16}$ {\scriptsize $[+0.06, +0.26]$} & $\mathbf{+0.24}$ {\scriptsize $[+0.15, +0.34]$} & $\mathbf{+0.28}$ {\scriptsize $[+0.19, +0.37]$} \\
Gemma-4-31B    & $\mathbf{-0.29}$ {\scriptsize $[-0.39, -0.19]$} & $-0.07$ {\scriptsize $[-0.17, +0.02]$} & $+0.08$ {\scriptsize $[-0.01, +0.17]$} & $\mathbf{+0.13}$ {\scriptsize $[+0.03, +0.23]$} \\
Gemini 2.5 Pro & $\mathbf{-0.38}$ {\scriptsize $[-0.49, -0.27]$} & $\mathbf{-0.21}$ {\scriptsize $[-0.32, -0.10]$} & $+0.05$ {\scriptsize $[-0.03, +0.13]$} & $+0.07$ {\scriptsize $[-0.02, +0.16]$} \\
\bottomrule
\end{tabular}
\label{tab:paired_ci}
\end{table}

\paragraph{Significance Tests.}
Beyond interval estimation, we report McNemar's test in
Table~\ref{tab:mcnemar} to examine whether the observed accuracy differences could have
arisen by chance. The test considers only the questions that exactly one of the two
systems answers correctly and compares each agentic framework against Single-Turn in
each of the 16 open-weight settings in Table~\ref{tab:main_results}. These disagreement
questions fall into two cases: those answered correctly only by the agentic framework,
and those answered correctly only by Single-Turn. A large imbalance between the two
cases establishes a clear ordering between the systems, allowing us to check whether the
performance gap between the two approaches emerges as statistically significant. By pooling
the 16 open-weight models for each agentic framework and testing the imbalance among the
disagreement questions, we observe that all four frameworks yield a $p$ value below $0.001$, indicating a
substantive performance difference between each framework and Single-Turn. This indicates
that the performance differences observed in Table~\ref{tab:main_results} do not arise
only for a particular backbone but stem from the framework design itself, and that
AgentVidBench captures these differences reliably.

\begin{table}[h]
\centering
\caption{McNemar's test comparing each agentic framework against Single-Turn, pooled over the
16 open-weight settings in Table~\ref{tab:main_results}.}
\small
\setlength{\tabcolsep}{5.5pt}
\renewcommand{\arraystretch}{1.12}
\begin{tabular}{@{}l cc@{}}
\toprule
\textbf{Comparison} & \textbf{Settings resolved} & \textbf{Pooled $p$} \\
\midrule
Ours vs.\ Single-Turn & 15 / 16 & $<0.001$ \\
AVP~\cite{wang2025activevideoperceptioniterative} vs.\ Single-Turn  & 15 / 16 & $<0.001$ \\
DVD~\cite{DBLP:conf/nips/ZhangJGLLLL25} vs.\ Single-Turn  & \phantom{0}10 / 16 & $<0.001$ \\
STAR~\cite{fan2025toolaugmented} vs.\ Single-Turn & \phantom{0}12 / 16 & $<0.001$ \\
\bottomrule
\end{tabular}
\label{tab:mcnemar}
\end{table}

\newpage

\subsection{Model-ranking Stability}
\label{app:ranking-stability}

To examine how much the model ranking induced by the benchmark varies
with the choice of questions, we report ranking stability estimates in
Table~\ref{tab:rank_stability}. Specifically, we reuse the bootstrap samples of
Section~\ref{app:framework-comparisons} and re-rank all 14 models on every sample. A
model's score is defined as its mean accuracy over the five methods, so that the ranking
reflects performance on the benchmark as a whole rather than under a single framework.
From the resulting distribution of ranks, we obtain a 95\% rank interval for each model
together with its probability of being ranked first and of being placed in the top three.
If the ranks do not fluctuate substantially across samples, AgentVidBench can be regarded
as separating models by performance level reliably. The rank correlation across samples,
measured by Kendall's $\tau$-b, is $0.83$, so the ranking is largely preserved when the
questions change. Each of the top four models places in the top three across most samples
(61.7\%--97.1\%), whereas the remaining ten models reach the top three in at most 3.3\%
of samples. The boundary between the two groups is thus clearly maintained, demonstrating that AgentVidBench provides a robust and stable stratification of model capabilities.

\begin{table}[h]
\centering
\caption{Bootstrap stability of the model ranking induced by AgentVidBench. A model's score is its
mean accuracy over the five methods in Table~\ref{tab:main_results}, and all 16 models are re-ranked
on every bootstrap sample.}
\small
\setlength{\tabcolsep}{5.5pt}
\renewcommand{\arraystretch}{1.12}
\begin{tabular}{@{}l ccc@{}}
\toprule
\textbf{Models} & \textbf{95\% rank interval} & \textbf{P(rank 1)} & \textbf{P(top 3)} \\
\midrule
Gemma-4-31B \textit{+Think}     & $[1, 3]$  & 71.2\% & 98.1\% \\
Gemma-4-31B                     & $[1, 5]$  & 14.3\% & 79.0\% \\
Gemma-4-26B-A4B \textit{+Think} & $[1, 6]$  & 11.5\% & 72.9\% \\
Qwen3.5-27B \textit{+Think}     & $[2, 6]$  & \phantom{0}2.3\% & 35.0\% \\
\midrule
Remaining 10 models             & $[2, 16]$ & $\leq$0.5\% & $\leq$9.0\% \\
\bottomrule
\end{tabular}
\label{tab:rank_stability}
\end{table} 
\section{Appendix -- Controlled Comparison Between Agentic Frameworks}
\label{app:controlled-comparison}

\begin{table}[t!]
\centering
\caption{Single-Turn results under the original 1 fps setting and under a
matched 50-frame setting. Arrows indicate the change relative to the 1 fps setting.}
\small
\setlength{\tabcolsep}{5.5pt}
\renewcommand{\arraystretch}{1.12}
\begin{tabular}{@{}l cc cc@{}}
\toprule
\multirow{2}{*}{\textbf{Models}}
& \multicolumn{2}{c}{\textbf{1 fps}}
& \multicolumn{2}{c}{\textbf{50 frames}} \\
\cmidrule(lr){2-3}\cmidrule(l){4-5}
& Acc. & Traj. & Acc. & Traj. \\
\midrule
Gemini 2.5 Pro       & 0.51 & 0.71 & 0.40 \textcolor{red}{\scriptsize{$\downarrow$0.11}} & 0.68 \textcolor{red}{\scriptsize{$\downarrow$0.03}} \\
Qwen3.5-27B          & 0.22 & 0.53 & 0.29 \textcolor{blue}{\scriptsize{$\uparrow$0.07}} & 0.57 \textcolor{blue}{\scriptsize{$\uparrow$0.04}} \\
\quad\textit{+Think} & 0.25 & 0.57 & 0.23 \textcolor{red}{\scriptsize{$\downarrow$0.02}} & 0.47 \textcolor{red}{\scriptsize{$\downarrow$0.10}} \\
Gemma-4-31B          & 0.38 & 0.59 & 0.26 \textcolor{red}{\scriptsize{$\downarrow$0.12}} & 0.51 \textcolor{red}{\scriptsize{$\downarrow$0.08}} \\
\quad\textit{+Think} & 0.37 & 0.61 & 0.24 \textcolor{red}{\scriptsize{$\downarrow$0.13}} & 0.50 \textcolor{red}{\scriptsize{$\downarrow$0.11}} \\
\bottomrule
\end{tabular}
\label{tab:video_input}
\end{table}

In this section, we control the main confounding factors to examine
whether the differences in Table~\ref{tab:main_results} stem from the framework design
itself. Specifically, Section~\ref{app:video-input} matches the video input across all
models in the single-turn evaluation, while Section~\ref{app:tool-budget} aligns the
video recognition tools of the baselines with ours.

\subsection{Video Input}
\label{app:video-input}

In the single-turn evaluation of Table~\ref{tab:main_results}, the
50-frame setting was applied only to GPT-5 and Claude Opus 4.7, following the protocol of
prior work\,\cite{DBLP:journals/corr/abs-2604-05015}, since these models cannot structurally ingest an entire long
video at 1 fps. To compare performance under a common condition, we report in
Table~\ref{tab:video_input} the results of re-evaluating a subset of the models under the
same 50-frame setting. Single-Turn accuracy is lower under this setting than at 1 fps for
four of the five models. This supports the need for an agentic approach that inspects
only the relevant segments, particularly for models whose structural constraints prevent
them from processing a long video in a single pass.

\subsection{Tool Capability and Token Budget}
\label{app:tool-budget}

To rule out the possibility that the performance gains stem from stronger
external tools or a larger compute budget rather than from the framework design, we report
in Table~\ref{tab:tool_budget} the results of replacing the video recognition tools of DVD
and STAR with the same Gemini 2.5 Pro model used by our method. Even after the tools are aligned,
our method retains the highest accuracy on both backbones, while the accuracy of DVD and
STAR decreases slightly in some settings. The average number of tokens per question is also
not monotonically related to accuracy. This suggests that the observed performance
differences are difficult to explain by tool capability or compute budget alone.

\begin{table}[h]
\centering
\caption{Results after aligning the video recognition tools of DVD and
STAR with the Gemini 2.5 Pro model used by our method. Single-Turn rows use the matched
50-frame setting of Table~\ref{tab:video_input}; arrows indicate the change relative to
Table~\ref{tab:main_results}. Tok./Q denotes the average number of tokens per question.}
\small
\setlength{\tabcolsep}{5.5pt}
\renewcommand{\arraystretch}{1.12}
\begin{tabular}{@{}l ccc@{}}
\toprule
\textbf{Method} & \textbf{Acc.} & \textbf{Traj.} & \textbf{Tok./Q} \\
\midrule
\rowcolor{gray!7}
\multicolumn{4}{@{}l}{\quad\textit{Qwen3.5-27B}} \\
Single-Turn           & 0.29 & 0.57 & 14.2K \\
STAR + Gemini 2.5 Pro & 0.08 \textcolor{red}{\scriptsize{$\downarrow$0.01}} & 0.43 \textcolor{blue}{\scriptsize{$\uparrow$0.17}} & 40.3K \\
DVD + Gemini 2.5 Pro  & 0.36 \textcolor{red}{\scriptsize{$\downarrow$0.01}} & 0.64 \textcolor{blue}{\scriptsize{$\uparrow$0.06}} & 149.3K \\
Ours                  & \textbf{0.49} & \textbf{0.78} & 158.8K \\
\midrule
\rowcolor{gray!7}
\multicolumn{4}{@{}l}{\quad\textit{Gemma-4-31B}} \\
Single-Turn           & 0.26 & 0.51 & \phantom{00}3.7K \\
STAR + Gemini 2.5 Pro & 0.25 \textcolor{blue}{\scriptsize{$\uparrow$0.17}} & 0.42 \textcolor{blue}{\scriptsize{$\uparrow$0.14}} & 27.6K \\
DVD + Gemini 2.5 Pro  & 0.45 \textcolor{blue}{\scriptsize{$\uparrow$0.14}} & 0.65 \textcolor{blue}{\scriptsize{$\uparrow$0.06}} & 113.5K \\
Ours                  & \textbf{0.51} & \textbf{0.75} & 189.3K \\
\bottomrule
\end{tabular}
\label{tab:tool_budget}
\end{table}
\section{Appendix -- Performance on Existing Datasets}
\label{app:performance-existing-datasets}

We evaluate Single-Turn and our proposed baseline method, both with the Gemini 2.5 Pro backbone, on the Video-MME Long subset. Both achieve $\sim90\%$ accuracy, approaching the practical ceiling for this subset. As Video-MME questions are predominantly single-hop with the required evidence concentrated in 1--2 segments, the achievable gain from iterative tool use is upper-bounded. In contrast, AgentVidBench is built to require multi-hop reasoning over dispersed evidence (Sections~\ref{sec:benchmark-construction}--\ref{sec:statistics}). This is supported by Table~\ref{tab:ablation-modules}, where all three adaptive inspection controls contribute comparably on our benchmark.

\end{document}